\documentclass[12pt]{article}
\usepackage[utf8]{inputenc}
\usepackage{enumerate}
\usepackage{url} 
\usepackage{subcaption}
\usepackage{float}
\usepackage{setspace}
\usepackage[title]{appendix}
\usepackage[table,xcdraw,dvipsnames]{xcolor}
\usepackage[english]{babel}
\usepackage[round]{natbib}
\usepackage{booktabs}
\usepackage{setspace}
\usepackage[affil-it]{authblk}

\usepackage{titlesec}
\usepackage[ruled,vlined, linesnumbered]{algorithm2e} 
\usepackage{algpseudocode}
\usepackage{parskip}

\usepackage[colorlinks,citecolor=blue,urlcolor=magenta,bookmarks=false,hypertexnames=true]{hyperref} 
\usepackage{xspace}
\usepackage{multirow}
\usepackage{amsthm}
\usepackage{amssymb}
\usepackage{amsmath}
\usepackage{amsfonts}
\usepackage{thmtools}
\usepackage{mathrsfs}
\usepackage{mathtools}
\usepackage{bm}
\usepackage{bbm}
\usepackage{hyphenat}
\usepackage{caption,subcaption}
\usepackage{tikz}
\usepackage[T1]{fontenc}
\usepackage{lmodern}
\usepackage{graphicx}
\usepackage{blindtext}
\usepackage{verbatim}
\numberwithin{equation}{section}
\makeatother

\newtheorem{thm}{Theorem}[section]
\newtheorem{cor}{Corollary}[thm]
\newtheorem{lemma}[thm]{Lemma}
\newtheorem{prop}{Proposition}[section]
\theoremstyle{definition} 
\newtheorem{rmk}{Remark}[section]
\newtheorem{defn}{Definition}[section]

\DeclarePairedDelimiter{\ceil}{\lceil}{\rceil}

\newcommand{\PX}{\hat{P}^n_X}
\newcommand{\potnet}{\textsf{POTNet}\xspace}
\newcommand{\mpw}{\textsf{MPW}\xspace}
\newcommand{\PXj}{\hat{P}^n_{X,j}}

\newcommand{\X}{\mathcal{X}}
\newcommand{\R}{\mathbb{R}}
\newcommand{\N}{\mathcal{N}}
\newcommand{\W}{\mathbf{W}}
\newcommand{\h}{\mathbf{h}}
\DeclareMathOperator*{\argmax}{arg\,max}

\title{Efficient Generative Modeling via Penalized Optimal Transport Network}

\newcommand\CoAuthorMark{\footnotemark[\arabic{footnote}]} 
\author[1]{Wenhui Sophia Lu\footnote{Contributed equally}}

\author[3]{Chenyang Zhong\protect\CoAuthorMark}

\author[1,2]{Wing Hung Wong
  \thanks{Corresponding author; Electronic address: \texttt{whwong@stanford.edu}}}

\affil[1]{Department of Statistics, Stanford University}
\affil[2]{Department of Biomedical Data Science, Stanford University}
\affil[3]{Department of Statistics, Columbia University}

\date{}

\begin{document}

\maketitle
\vspace{-0.35in}
\begin{center}
\date{\today}
\end{center}

\begin{abstract}
The generation of synthetic data with distributions that faithfully emulate the underlying data-generating mechanism holds paramount significance across a wide array of scientific disciplines. 
Wasserstein Generative Adversarial Networks (WGANs) have emerged as a prominent tool for this task because of their remarkable flexibility. 
However, due to the delicate balance between the generator and discriminator networks and the instability of Wasserstein distance in high dimensions, WGAN often manifests the pathological phenomenon of \textit{mode collapse}.
This results in generated samples that converge to a restricted set of outputs and fail to adequately capture tail behaviors of the true distribution. 
Such limitations can lead to serious downstream consequences, potentially compromising the integrity and applicability of synthetic data in critical applications.

To this end, we propose the \textit{Penalized Optimal Transport Network} (\potnet), a versatile deep generative model based on a robust and interpretable marginally-penalized Wasserstein (\mpw) distance. Through the \mpw distance, \potnet effectively leverages low-dimensional marginal information to guide the overall alignment of joint distributions.
Furthermore, our primal-based framework enables direct evaluation of the \mpw distance, thus eliminating the need for a critic network.
This formulation circumvents training instabilities inherent in adversarial approaches and avoids the need for extensive parameter tuning.
Theoretically, we derive non-asymptotic bound on the generalization error of the \mpw loss and establish convergence rates of the generative distribution learned by \potnet. Our theoretical results elucidate the efficacy and reliability of the proposed model in mitigating mode collapse. 
Through comprehensive simulation studies and real-world data applications, we demonstrate the superior performance of \potnet in accurately capturing underlying data structures, including their tail behaviors and minor modalities.
Moreover, our model achieves orders of magnitude speedup during the sampling stage compared to state-of-the-art alternatives, which enables computationally efficient large-scale synthetic data generation. 
\end{abstract}

\noindent%

\section{Introduction}
The generation of synthetic data whose distributions closely mirror the true data generating mechanism holds significant importance across various disciplines, particularly for model evaluation, data augmentation, and model selection. 
When the true parameter of interest is unobserved, such as causal effects in econometrics, researchers 
typically rely on Monte Carlo studies with discretionary and idealized assumptions, which can undermine the real-world applicability of the proposed methods \citep{advani2019mostly, knaus2021machine}. 
In contrast, faithful synthetic data provide a means for researchers to systematically evaluate methods under conditions that more closely reflect real-world scenarios.
On the other hand, synthetic data are increasingly being used to train downstream models. For example, in Bayesian likelihood-free inference, researchers use synthetic data to directly approximate complex posterior distributions and bypass the need of explicit likelihood evaluations \citep{zhou2023deep}. In machine learning, augmenting training sets with synthetic data has been proposed as a solution to address class imbalance issues \citep{ganganwar2024employing}. 
Furthermore, practitioners may use synthetic datasets to evaluate the empirical performance of estimators on specific datasets, and utilize these performance metrics to guide the model selection process \citep{athey2024using}.
Generating high-quality artificial 
datasets is therefore vital for ensuring accurate downstream analyses and informed decision-making.

Deep generative models, including the family of Generative Adversarial Networks (GANs; \citealp{goodfellow2020generative}), have emerged as powerful tools for synthetic data generation.
The GAN formulation considers a minimax problem, where the \textit{generator} network aims to model the distribution of the data generating process, and the \textit{critic} network seeks to distinguish between real and fake data.
Yet, these models are susceptible to \textit{mode collapse}, a pathological phenomenon where the generator ``collapses'' to a restricted set of outputs, repeatedly generating similar data while neglecting tails of the original data distribution \citep{gulrajani2017improved, thanh2020catastrophic}. 
This limitation can lead to synthetic data that underrepresents or entirely omits minority subgroups in the training set, thereby exacerbating biases in downstream applications. 
Numerous studies have examined the profound and far-reaching social ramifications of mode collapse, particularly its potential to produce inequitable outcomes in critical domains such as hiring algorithms, facial recognition, and criminal justice risk assessments \citep{gao2022fairneuron, jain2022imperfect, shumailov2024ai}. 
Moreover, GANs often require extensive parameter tuning, such as determining the appropriate depth of the discriminator and generator networks. 
Despite these efforts, GANs remain notoriously difficult to train, with gradient descent-based optimization frequently failing to achieve convergence in practice \citep{mescheder2018training}. Even for the successor Wasserstein GANs (WGANs), which 
attempts to alleviate the problem of vanishing gradients by using Wasserstein metric, mode collapse remains a persistent challenge \citep{pei2108alleviating}.

Empirical studies and theoretical analyses suggest that mode collapse arises from the saddle-point formulation of the adversarial training process and the highly non-convex optimization landscape, both of which lead to a number of instabilities \citep{gulrajani2017improved, mescheder2018training}.
There have been substantial efforts made to address these issues, such as introducing gradient clipping \citep{arjovsky2017wasserstein} or penalty terms on the discriminator’s gradient \citep{gulrajani2017improved}.
However, empirical observations frequently reveal that, even when the optimization process ostensibly succeeds, the resultant generator often fails to ensure good generalization properties or sufficient diversity in the generated samples \citep{arora2017gans}.
As a direct consequence of these shortcomings, mode collapse and training instabilities persist as critical pathologies. These issues are particularly pronounced in higher-dimensional settings, where the empirical convergence rate of the Wasserstein distance deteriorates exponentially with increasing dimensionality \citep{fournier2015rate}.

Furthermore, tabular data, rectangular arrays of values where each column represents a well-defined attribute and each row corresponds to a distinct record, often exhibit multimodality and consist of heterogeneous features, including categorical and continuous variates. These intricate structures pose additional challenges for GANs and WGANs, which were originally designed for data with homogeneous feature types such as images and text.
As tabular data are ubiquitous across the natural and social sciences, as well as clinical studies, researchers may seek some degree of interpretability in the complex generative models used for synthetic data generation \citep{garcia2023evaluation, valina2024explainable}. For instance, they may wish to evaluate how the model penalizes mismatches between synthetic and real data for a specific covariate, such as \textit{sex} or \textit{cholesterol level}, or prioritize the alignment of certain features while placing less emphasis on others.

Motivated by the aforementioned challenges, we propose the \textit{Penalized Optimal Transport Network} (\potnet), a versatile generative model for synthetic data generation based on an intuitively appealing \textit{Marginally-Penalized Wasserstein} (\mpw) loss. 
Although \potnet is compatible with various data types including image data, the majority of this paper focuses on its utility in the context of tabular data distributions.
Our model consists of two core components: a data transformer that standardizes and encodes data of potentially mixed data types to a uniform scale, 
and a generator network that learns a pushforward map approximating the target measure optimized under the \mpw loss. 
The generator network has a natural conditional extension that enables conditioning on an arbitrary subset of features. This capability is particularly useful for data imputation, whereby users can generate synthetic data as a proxy for missing data by conditioning on the observed features.
The main advantages of the proposed framework and a summary of our contributions are presented below. 
First, unlike approaches that utilize the Kantorovich-Rubinstein duality, our primal-based framework allows for direct evaluation of the \mpw distance via a single minimization step. This formulation eliminates the need for a critic network, thereby fully circumventing potentially undesirable equilibria associated with sharp gradients in discriminator functions during minimax optimization. 
Second, we develop the \mpw distance, which explicitly leverages low-dimensional marginal information to guide the alignment of high-dimensional joint distributions. 
This novel loss function penalizes marginal mass misallocation, thereby effectively accommodating mixed data types and ensuring a fast convergence rate for marginal distributions. The latter is crucial in empirical settings, where interpretable parameters and the focus of inference often concentrate on marginals.
Third, the \mpw loss and \potnet offer desirable theoretical guarantees. We derive non-asymptotic bound on the generalization error of the \mpw loss and establish convergence rates of the generative distribution learned by \potnet to the true data-generating distribution in terms of both marginal and joint distributions. Moreover, through both theoretical and empirical analyses, we demonstrate how \potnet effectively addresses the mode collapse issue, including both mode dampening and tail shrinkage. 

The rest of the the paper is organized as follows. 
Section \ref{Sec:prelim} provides a brief review of the relevant background.
Section \ref{Sec:method} introduces the Marginally-Penalized Wasserstein distance and presents the Penalized Optimal Transport Network methodology. 
Section \ref{Sec:theory} presents a comprehensive theoretical analysis of the \mpw distance and \potnet.
In Section \ref{Sec:sim}, we assess the effectiveness of \potnet through extensive simulation studies. 
In Section \ref{Sec:app}, we illustrate the applicability of the proposed method using four real-world datasets.
Section \ref{Sec:disc} offers concluding remarks and discussion. Technical details, proofs of theoretical results from Section \ref{Sec:theory}, and additional simulation results are provided in the Appendix.
 
\section{Preliminaries}\label{Sec:prelim}
In this section, we provide a concise overview of the relevant preliminaries.

\paragraph{Notation}
We denote by $\|\cdot\|_2$ the Euclidean norm. The following notations apply to any $m\in\mathbb{N}_{+}$. We denote $[m]:=\{1,2,\cdots,m\}$, and denote by $\mathbb{S}_m$ the set of all permutations of size $m$. We denote by $\mathrm{Lip}_{1,m}$ the set of $1$-Lipschitz functions on $\mathbb{R}^m$, i.e., real-valued functions $f$ on $\mathbb{R}^m$ such that $|f(\mathbf{x})-f(\mathbf{y})|\leq \|\mathbf{x}-\mathbf{y}\|_2$ for all $\mathbf{x},\mathbf{y}\in\mathbb{R}^m$. For any $\mathbf{x}_0\in \mathbb{R}^m$ and $r>0$, we define $B_r(\mathbf{x}_0):=\{\mathbf{x}\in\mathbb{R}^m:\|\mathbf{x}-\mathbf{x}_0\|_2\leq r\}$; when $\mathbf{x}_0=(0,0,\cdots,0)\in\mathbb{R}^m$, we write $B_r(\mathbf{x}_0)$ as $B_r$. For any $\mathbf{a}\in\mathbb{R}^m$, we denote by $\delta_{\mathbf{a}}$ the Dirac measure at $\mathbf{a}$.  

\paragraph{Implicit distributions and generative models}
Implicit distributions are defined as probability models from which samples can be easily drawn but whose density functions are not directly available. 
Generative models are prominent examples as they learn a representation of the underlying data distribution without producing explicit estimates of the density functions.
Let $Z \sim P_Z(\cdot)$, where $Z \in \mathcal{Z} \subseteq \mathbb{R}^{d_z}$, represent a low-dimensional noise variable typically drawn from an isotropic Gaussian distribution $\mathcal{N}(0, \mathbf{I}_{d_z})$, and let $P_X(\cdot)$ denote the true data-generating distribution. 
Denote the observed data by $X^{(1)},\cdots,X^{(n)}\in\mathcal{X}\subseteq \mathbb{R}^d$, which are generated i.i.d. from $P_X(\cdot)$, and let $\pmb{X}:=(X^{(1)},\cdots,X^{(n)})$. Implicit models aim to learn a deterministic pushforward mapping $G_\theta: \mathcal{Z} \to \mathcal{X}$, $\theta\in\Theta$, which is typically parameterized by a deep neural network defined later in this section. This mapping is optimized to minimize a well-defined notion of discrepancy between the distribution $P_\theta$ of the generated samples $G_\theta(Z)$ and the empirical measure of real observations, $\PX := n^{-1} \sum_{i=1}^n \delta_{X^{(i)}}$. 
We will refer to $G_\theta$ as the \textit{generator} network and $P_\theta$ as the \textit{generative distribution} henceforth. 
In the conventional GAN framework, discrepancy (e.g., Kullback-Leibler divergence or Wasserstein distance) is evaluated using a \textit{discriminator} or \textit{critic} network without requiring an analytical expression of the densities. Given a batch of synthetic samples and real observations, the discriminator aims to differentiate between real and fake samples \citep{goodfellow2020generative}. Thus the full GAN objective is given by $\min\limits_{\theta\in\Theta} \max\limits_{f\in\mathcal{F}}\big\{\mathbb{E}_{X \sim \PX}[f(X)] - \mathbb{E}_{Z \sim P_Z}[f(G_{\theta}(Z))]\big\}$,
where $\mathcal{F}$ is a suitable class of functions. When $\mathcal{F}$ consists of the family of $1$-Lipschitz functions on $\X$, the resulting GAN is the Wasserstein GAN (WGAN) \citep{arjovsky2017wasserstein}. Empirically, the adversarial training process often requires striking a subtle balance between the complexity and training dynamics of the generator and discriminator networks \citep{gulrajani2017improved}.

\paragraph{Multilayer feedforward neural network}
To estimate the pushforward mapping to the target distribution, we employ a generator network structured as a fully connected feedforward neural network. This network is composed of recursive layers, each consisting of an affine transformation parameterized by weights $\W$ followed by elementwise application of the ReLU activation function, defined as $\sigma(x) = \max\{x, 0\}$.
Let $L$ denote the \textit{number of hidden layers}, and let $\h = (h_0, h_1, \dots, h_L)$ specify the number of hidden neurons in each of the $L+1$ layers. By convention, $h_0$ represents the input dimension, while $h_{L+1}$ indicates the output dimension. A multilayer feedforward neural network with architecture $(L, \h)$ and ReLU activation $\sigma(\cdot)$ defines a real-valued function mapping $\mathbb{R}^{h_0} \to \mathbb{R}^{h_{L+1}}$, expressed as $g(\mathbf{x}) = f_{L+1} \circ \sigma \circ f_L \circ \sigma \cdots \circ \sigma \circ f_1(\mathbf{x})$,
where each layer is represented by a linear transformation.
It is important to note that two neurons are connected if and only if they belong to adjacent layers (hence the name \textit{fully connected}).

\paragraph{Wasserstein distance}

Now we briefly review the background on the Wasserstein distance. Fix a Borel set $\mathcal{X}\subseteq\mathbb{R}^d$. For any $p\geq 1$, we denote by $\mathcal{P}_p(\mathcal{X})$ the set of probability distributions $\mu$ on $\mathcal{X}$ with finite $p$th moment, i.e., $\int_{\mathcal{X}}\|\mathbf{x}\|_2^p d\mu(\mathbf{x})<\infty$. For any probability distributions $\mu,\nu\in\mathcal{P}_p(\mathcal{X})$, the \emph{$p$-Wasserstein distance} between them is defined as 
\begin{equation}\label{PD}
    W_{p}(\mu,\nu):= \Big(\inf_{\pi\in\Pi(\mu,\nu)}\int_{\mathcal{X}\times \mathcal{X}} \|\mathbf{x}-\mathbf{y}\|_2^p d\pi(\mathbf{x},\mathbf{y})\Big)^{1\slash p},
\end{equation}
where $\Pi(\mu,\nu)$ is the set of probability distributions on $\mathcal{X}\times\mathcal{X}$ with first and second marginals given by $\mu$ and $\nu$, respectively. It can be shown that the $W_p(\cdot,\cdot)$ gives a valid metric on $\mathcal{P}_p(\mathcal{X})$ \cite[Chapter 6]{villani2009optimal}. The form \eqref{PD} of the $p$-Wasserstein distance is closely connected to optimal transport theory \citep{villani2009optimal}: if the cost of transporting one unit of mass from $\mathbf{x}$ to $\mathbf{y}$ (where $\mathbf{x},\mathbf{y}\in\mathcal{X}$) is given by $\|\mathbf{x}-\mathbf{y}\|_2^p$, then $\inf\limits_{\pi\in\Pi(\mu,\nu)}\int_{\mathcal{X}\times \mathcal{X}} \|\mathbf{x}-\mathbf{y}\|_2^p d\pi(\mathbf{x},\mathbf{y})$ corresponds to the minimum transport cost for moving the mass from distribution $\mu$ to distribution $\nu$.

The $p$-Wasserstein distance allows the following dual formulation \cite[Theorem 5.10]{villani2009optimal}: for any $\mu,\nu\in\mathcal{P}_p(\mathcal{X})$, 
\begin{equation*}
    W_{p}(\mu,\nu) = \sup\limits_{(\psi,\phi)\in\mathcal{G}_p}\Big( \int_{\mathcal{X}}\phi d\nu- \int_{\mathcal{X}}\psi d\mu \Big)^{1\slash p},
\end{equation*}
where $\mathcal{G}_p$ is the set of pairs of bounded continuous functions $(\psi,\phi)$ on $\mathcal{X}$ that satisfy $\phi(\mathbf{y})-\psi(\mathbf{x})\leq \|\mathbf{x}-\mathbf{y}\|_2^p$ for all $\mathbf{x},\mathbf{y}\in \mathcal{X}$. By Kantorovich-Rubinstein duality \cite[Case 5.16]{villani2009optimal}, when $p=1$,
\begin{equation*}
 W_{1}(\mu,\nu)=\sup\limits_{f\in \mathrm{Lip}_{1}(\mathcal{X})}\Big\{\int_{\mathcal{X}}f d\mu-\int_{\mathcal{X}}f d\nu\Big\},
\end{equation*}
where $\mathrm{Lip}_{1}(\mathcal{X})$ is the set of $1$-Lipschitz functions on $\mathcal{X}$.

\section{Penalized Optimal Transport Network}\label{Sec:method}
In this section, we introduce the Penalized Optimal Transport Network (\potnet), a novel generative model designed to address the challenges of training instabilities and mode collapse frequently encountered in existing approaches.
We highlight two important characteristics of \potnet: (1) it eliminates the need for a critic network, and (2) it utilizes the \mpw loss which incorporates marginal regularization. These distinctions significantly reduce the need of \potnet for extensive hyperparameter tuning, making it a more robust and stable alternative to current generative models.

\subsection{Marginally-Penalized Wasserstein Distance}
The caliber of synthetic samples hinges on the divergence used to quantify proximity between generative and target distributions.
A well-known limitation of the Wasserstein distance is its exponentially deteriorating convergence rate as the dimensionality increases \citep{fournier2015rate}. 
However, when joint distributions are aligned, their marginal distributions along each coordinate axis must coincide.
Building on this observation, we propose the Marginally-Penalized Wasserstein distance (\mpw), which leverages the faster convergence of marginal estimations to enhance the overall alignment of joint distributions.
By construction, the \mpw distance inherently accommodates heterogeneous data types. To lay the groundwork for the proposed framework, we first present key definitions and results necessary for subsequent discussions.

\begin{defn}[Marginal Distribution]\label{def_marginal}
For any $j\in [d]$, we define the projection map $p_j:\mathbb{R}^d\rightarrow\mathbb{R}$ by setting $p_j(\mathbf{x}):= x_j$ for any $\mathbf{x}=(x_1,x_2,\cdots,x_d)\in\mathbb{R}^d$. For any probability distribution $\mu$ on $\mathbb{R}^d$, we denote by $(p_j)_{*}\mu$ the pushforward of $\mu$ by $p_j$, i.e., the probability distribution on $\mathbb{R}$ such that for any Borel set $A\subseteq \mathbb{R}$, $(p_j)_{*}\mu(A)=\mu(p_j^{-1}(A))$, where $p_j^{-1}(A):=\{\mathbf{x}=(x_1,x_2,\cdots,x_d)\in\mathbb{R}^d:x_j\in A\}$. 
\end{defn}

\begin{defn}[Marginally-Penalized Wasserstein (\mpw) Distance]\label{def:mpw-dist}
Recall the notations in Section \ref{Sec:prelim}. For any two probability distributions $\mu,\nu\in\mathcal{P}_1(\mathcal{X})$, we define the \emph{marginally-penalized Wasserstein} (\mpw) \emph{distance} between $\mu$ and $\nu$ as
\begin{align}
    \mathcal{D}(\mu,\nu):= \underbrace{W_1(\mu,\nu)}_{\text{joint distance}}+ \sum_{j=1}^d \underbrace{\lambda_j W_1((p_j)_{*}\mu,(p_j)_{*}\nu)}_{\text{marginal penalty}}, 
    \label{mpw-loss-def}
\end{align}
where $\pmb{\lambda} = (\lambda_1,\lambda_2,\cdots,\lambda_d)$ are hyperparameters.
\end{defn}

\begin{rmk}
    The definition above corresponds to the specific case of the $1$-coordinate \mpw distance chosen for its conceptual clarity and notational brevity. The \textit{Generalized} \mpw distance, computed using a general set of marginals, is detailed in Appendix A.
\end{rmk}

The \mpw distance is a metric on $\mathcal{P}_1(\mathcal{X})$ and allows a dual representation; we refer to Appendix B for details. For discussions on the conceptual motivations underpinning the \potnet framework, including a rationale for the \mpw distance from the perspective of constrained optimization, see Section \ref{Sec:theory}. We also demonstrate in that section that it is theoretically possible to mitigate the curse of dimensionality by accurately estimating low-dimensional marginal distributions.

We now briefly discuss the practical significance of marginal distributions and how the \mpw distance naturally accommodates mixed data types.
\paragraph{Practical significance and interpretability of marginal estimation}
From an empirical perspective, accurate marginal estimation and effective recovery of marginal characteristics are highly important.
Researchers in natural and social sciences frequently focus on quantities derived from marginal distributions due to their meaningful interpretations and practical significance \citep{open2015estimating, serreze2015arctic, pew2015american}. 
For example, econometricians are interested in identifying income composition and pay particular attention to the top tail of the heavy-tailed income distribution \citep{piketty2003income}.
On the other hand, Bayesian researchers frequently report inferences for individual parameters in posterior estimation \citep{drovandi2024improving}. 
Moreover, marginal distributions may inform model identifiability during initial modeling phases.
Precise characterization of the marginals such as their tails is therefore crucial for constructing reliable posterior intervals.

\paragraph{Accommodation of heterogeneous data modalities}
The diverse distributional structures of continuous and categorical variates in tabular data pose a considerable challenge for generative modeling in this domain \citep{xu2019modeling, gorishniy2021revisiting, shwartz2022tabular}. 
By its definition, the \mpw distance is unaffected by this issue:
the joint loss remains well-defined for features normalized to a common scale, while the marginal penalty is inherently unaffected as it operates on each dimension separately (rather than employing stochastic projections on features with disparate typologies).

\paragraph{Related works}
Our proposed distance is broadly related to the Sliced Wasserstein (SW) distance, as both leverage low-dimensional projections of high-dimensional distributions. The SW distance is defined as the integral of Wasserstein distances between radial projections of input distributions onto random directions over the unit sphere.
However, our marginal penalty serves as a form of regularization in generative model training whereas the SW distance was developed to mitigate the computational complexity of Wasserstein minimization for high-dimensional distributions \citep{bonneel2015sliced, kolouri2018sliced}. Furthermore, the \mpw distance emphasizes original coordinate axes and their marginal distributions, which typically carry particular significance and yield meaningful interpretations in many practical applications. Given that researchers often prioritize on the validity of these marginal distributions, the \mpw distance offers a valuable advantage by providing explicit guarantees on these distributions.
Our proposed method also shares conceptual similarities with marginal adjustment in likelihood-free Bayesian inference \citep{drovandi2024improving}.
These studies propose estimating low-dimensional marginal posterior distributions based on the assumption that each parameter depends on only a few key summary statistics. In contrast, our approach uses the low-dimensional marginals to guide the alignment of joint distribution in generative model training. 
Nonetheless, this connection underscores the broad applicability and significance of leveraging marginal information for distributional inference.

\subsection{The \potnet Model}
\potnet performs generative modeling through direct estimation of the \mpw distance, thereby eliminating the need for a critic network.
It comprises two components. The \textit{data transformation} component encodes and maps features of potentially mixed data types in tabular data (i.e., continuous, categorical) to a common scale. 
The \textit{generator neural network} component approximates the pushforward map that transforms low-dimensional noise into the target distribution.
Parameters of the generator network are optimized by minimizing the \mpw loss, as detailed in Section \ref{optim}. We note that the current formulation of \potnet naturally extends to conditional data generation based on an arbitrary subset of covariates.
Details of Conditional \potnet are provided in Appendix C.

\subsubsection{Evaluating the \mpw Distance via Primal Formulation}
The relevant Wasserstein distances in \potnet are computed using the \emph{primal formulation} in \eqref{PD}.
Specifically, we need to evaluate the \mpw distance $\mathcal{D}(\mu,\nu)$ between two probability distributions $\mu$ and $\nu$ that are in the form of empirical distributions of $m$ samples.
We do so through a combination of primal evaluation and marginal sorting, which results in a total complexity of $O(m^2(m+d))$. We compute the joint distance component $W_1(\mu,\nu)$ by solving the optimal transport problem in its primal formulation \eqref{PD} (with $p=1$). We use the algorithm from \cite{bonneel2011displacement}, as implemented in the \texttt{POT} (Python Optimal Transport) library \citep{flamary2021pot}.
The complexity for computing the optimal transport plan is $O(m^3)$ \citep{kuhn1955hungarian, peyre2019computational}, and the complexity for computing $W_1(\mu,\nu)$ based on this optimal transport plan is $O(m^2 d)$. Hence the total computational complexity for evaluating $W_1(\mu,\nu)$ is $O(m^2(m+d))$. 
For each univariate marginal, calculating Wasserstein distances using empirical distributions reduces to sorting, making the task computationally efficient with a time complexity of $O(m\log{m})$.

It has been widely observed empirically that GANs and WGANs typically require extensive parameter tuning (e.g., selecting a suitable network architecture and gradient penalty parameter) due to the adversarial formulation \cite{gulrajani2017improved}. 
By eliminating the critic network, \potnet reduces the optimization problem to only optimizing parameters of the pushforward mapping.
Consequently, \potnet bypasses the need for extensive parameter tuning and improves training stability.
Through comprehensive simulation experiments, we demonstrate the robust performance and rapid convergence of \potnet in Section \ref{Sec:sim}.

\subsubsection{Data Transformation}
There are many subtle yet important distinctions between tabular data and structured data such as images. A significant difference lies in the heterogeneous nature of tabular data, which often comprises a mix of continuous, ordered discrete (ordinal), and categorical features. Properly representing this diverse data structure is crucial for both the training of generative models and the quality of the resulting synthetic datasets.

\paragraph{Categorical variables} 
Let $I_c$ denote the set of indices corresponding to categorical features.
Each categorical feature lacking inherent ordering is transformed using one-hot encoding, a technique that converts categorical variables into a set of binary indicators where 1 and 0 respectively represent the presence or absence of a specific category. Formally, let $X_j, ~j \in I_c$ be a categorical random variable with support $\{1, \dots, {\kappa_j}\}$, where ${\kappa_j}$ denotes the number of categories. The one-hot representation for $X_j$ is a binary vector ${\pmb{X}}^\dagger_j \in \{0,1\}^{\kappa_j}$, where exactly one element is 1 and the rest are 0. For the final activation layer of the generator model, we apply the softmax activation function which outputs a probability distribution in the ${\kappa_j}$-dimensional probability simplex $\Delta^{\kappa_j}$, defined as $\sigma(\pmb{z})_{\ell} = \frac{\exp(z_\ell)}{\sum_{\ell'=1}^{\kappa_j} \exp(z_{\ell'})}$, for $\ell \in [{\kappa_j}]$.
We remark that \potnet, through the elimination of the critic network, entirely circumvents instabilities engendered by the distinction between generated soft-encodings ($[0, 1]^{\kappa_j} \in \Delta^{\kappa_j}$) and real hard encodings ($\{0, 1\}^{\kappa_j}$). Our approach not only avoids the instability and suboptimal performance often encountered in existing GAN models when handling categorical data, but also obviates the need for the Gumbel-Softmax reparameterization typically used by GANs.
During the sample generation phase (but not during training), we apply the $\argmax$ function to the generated probabilities for each categorical variable indexed by $j \in I_c$, transforming them from $[0,1]^{\kappa_j}$ to $[\kappa_j]$.
We then use inverse one-hot encoding to obtain the respective categorical values.

\paragraph{Discrete and continuous variables}
Ordinal categorical variables possess an intrinsic ordering among their categories. For these variables, we map their original values to nonnegative integers and subsequently treat them as discrete features.
We do not differentiate between discrete and continuous features during the preprocessing and training stages. For all variables, we apply min-max scaling and normalize them to the range $[0, 1]$. This transformation ensures that all features, including one-hot encoded categorical variables, are uniformly scaled within the same interval.
The final dense layer of the generator network remain unmodified, with no additional activation function applied to either continuous or discrete features.
For generated variables corresponding to discrete columns, we subsequently round them to the nearest integer value.

\subsubsection{Generator Neural Network}\label{generator}
In \potnet, the generator neural network serves as a functional approximator that transforms low-dimensional noise into the target distribution. The generator consists of $L$ dense hidden layers with each dense layer $\ell$ comprising $h_\ell$ units. 
Following convention, we adopt ReLU activation function following each hidden layer, with a crucial exception at the final layer to preserve the full spectrum of generated values.
As previously described, softmax activation is applied on the output columns corresponding to categorical variable to yield probability distributions over their respective categories.
To enhance model stability and robustness, each nonlinear activation layer is followed by batch normalization and dropout layers.
Batch normalization standardizes activated neurons to zero mean and unit variance, effectively reducing internal covariate shift \citep{ioffe2015batch}, while
dropout randomly omits nodes and their connections during the training stage to reduce overfitting \citep{srivastava2014dropout}.

\subsubsection{Optimization Procedure}\label{optim}
We employ empirical risk minimization with the \mpw loss to train the generator network. Specifically, we aim to find the parameter $\hat{\theta}\in\Theta$ that solves the following minimization problem (recall the notations in Section \ref{Sec:prelim}):
\begin{equation}\label{theta_es}    \hat{\theta}:=\arg\min_{\theta\in\Theta} \mathcal{D}(\PX,P_{\theta}).
\end{equation}
We adopt a minibatch-based optimization procedure \citep{salimans2016improved, fatras2021minibatch} 
to make the Wasserstein minimization computationally feasible. 
This approach involves optimizing the expectation of the \mpw loss over minibatches using stochastic gradient updates. Each minibatch contains $m$ elements, except possibly the last one, which may have fewer if $n$ is not divisible by $m$. Commonly used batch sizes in relevant applications range from 32 to 512 \citep{arjovsky2017wasserstein, gulrajani2017improved}. 
Without loss of generality, we assume that $n$ is divisible by $m$ and each minibatch has size $m$. Specifically, we split the index set $[n]$ into $B=n\slash m$ blocks $\mathcal{I}_b:=\{i\in [n]:(b-1)m+1\leq i\leq bm\}$ for each $b\in[B]$. In each iteration $t\in [T]$ (where $T$ is the total number of iterations), we sample a uniformly random permutation $\pi_t\in\mathbb{S}_n$ and form minibatches of real observations $\mathcal{B}_X^{(t),b}:=\{X^{(\pi_t(i))}:i\in\mathcal{I}_b\}$ for each $b\in [B]$. Additionally, for each $b\in [B]$, we sample $\{Z^{(t),i}\}_{i\in\mathcal{I}_b}$ i.i.d. from the noise distribution $P_Z(\cdot)$ and form a minibatch of latent variables $\mathcal{B}_Z^{(t),b}:=\{Z^{(t),i}:i\in\mathcal{I}_b\}$. For any $b\in[B]$ and $\theta\in\Theta$, we denote by $\widehat{P}_X^{(t),b}$ the empirical distribution formed by $\mathcal{B}_X^{(t),b}$, and by $\widehat{P}_{G_{\theta}(Z)}^{(t),b}$ the empirical distribution formed by $\{G_{\theta}(Z^{(t),i}):i\in\mathcal{I}_b\}$. For each $b\in [B]$, the objective function for the $b$th minibatch in iteration $t$ is given by
\begin{align*}
    \mathcal{L}^{(t),b}(\theta):=\mathcal{D}(\widehat{P}_X^{(t),b},\widehat{P}_{G_{\theta}(Z)}^{(t),b})=W_1\left(\widehat{P}_X^{(t),b},\widehat{P}_{G_{\theta}(Z)}^{(t),b} \right)+\sum_{j=1}^d\lambda_jW_1\left((p_j)_{*}\widehat{P}_X^{(t),b}, (p_j)_{*}\widehat{P}_{G_{\theta}(Z)}^{(t),b}\right),
\end{align*}
for all $\theta\in\Theta$, where $(p_j)_{*}$ is defined as in Definition \ref{def_marginal}.

We employ the AdamW optimizer \citep{loshchilov2019decoupled}, an Adam variant that combines adaptive learning rates with enhanced weight decay regularization. To modulate learning rates, we utilize the cosine annealing scheduler. Detailed update rules are provided in Appendix E.

\subsubsection{Full Algorithmic Procedure}
The pseudo-code for \potnet is provided in Algorithm \ref{alg:POTNet}.

\begin{algorithm}[h!]
 \SetAlgoLined
\caption{Training Procedure for \potnet}
\label{alg:POTNet}
\KwIn{Data $\pmb{X}=(X^{(1)},\cdots,X^{(n)})$, latent dimension $d$, batch size $m$, regularization parameter $\pmb{\lambda} = (\lambda_1, \dots, \lambda_d)$.}
\KwOut{Optimized parameters $\theta$.}

\For{ iteration $t = 1, \dots, T$ }{
     \tcp*[h]{Permute indices}\; 
     $\pi_t: [n]\to[n] \in \mathbb{S}_n$\;
     \vspace{0.1cm}
    \For{$b = 1, \dots, \ceil{n/m}$}{
        Set $\mathcal{I}_b\gets\{i\in [n]:(b-1)m+1\leq i\leq bm\}$\; 
        Form minibatch of real observations $\mathcal{B}_X^{(t),b}\gets\{X^{(\pi_t(i))}:i\in\mathcal{I}_b\}$\; 
        Sample latent variables $\{Z^{(t),i}\}_{i\in \mathcal{I}_b}\sim P_Z(\cdot)$ and form minibatch $\mathcal{B}_Z^{(t),b}\gets\{Z^{(t),i}:i\in\mathcal{I}_b\}$\; 
        \vspace{0.1cm}
        \tcp*[h]{Compute gradient}\;
        $G_{\theta} \gets \nabla_\theta \left[ W_1\left(\widehat{P}_X^{(t),b},\widehat{P}_{G_{\theta}(Z)}^{(t),b} \right)+\sum_{j=1}^d\lambda_jW_1\left((p_j)_{*}\widehat{P}_X^{(t),b}, (p_j)_{*}\widehat{P}_{G_{\theta}(Z)}^{(t),b}\right)\right]$\;
        \vspace{0.1cm}
        \tcp*[h]{Update parameters}\;
         $\theta \gets \theta - \text{AdamW}(\theta, G_\theta)$\; 
    }
}
\end{algorithm}

\subsection{Improved Mode Retention and Diversity}
Traditional generative models, such as WGAN and those employing joint optimal transport loss, often struggle to accurately capture mode weights (Figure \ref{fig:gmm-bivariate_contours}) and tail behavior (Figure \ref{fig:mnist-ffnn}) in high dimensions.
The root cause of these pernicious behaviors lies in the curse of dimensionality affecting the Wasserstein distance: as dimensionality increases, the $W_1$ distance becomes overly sensitive to small perturbations and converges slowly at a rate of $n^{-1/d}$, where $d$ is the dimension. 

In contrast, the marginal penalty in \potnet effectively improves mode retention and sample diversity, as demonstrated in Theorems \ref{Thm_GMM}-\ref{Thm_T} and Corollary \ref{Cor1.1}. 
On a high-level, 
\potnet leverages the rapid convergence of marginal distributions, which occurs at a rapid rate of $n^{-1\slash 2}$. This faster convergence enables the marginal penalty to guide the generative distribution towards better alignment with the true data-generating distribution effectively.

Although mode collapse is a ubiquitously observed phenomenon, the term is often applied broadly to describe various undesired alterations of the target distribution. In the following discussion, we provide a rigorous characterization of two distinct types of mode collapse commonly occurring in practice. 
These two regimes will be referenced throughout subsequent sections of this article including theoretical analysis in Section \ref{Sec:theory} and experimental evaluation in Section \ref{Sec:sim}. 

\begin{defn}[Type I Mode Collapse]
Type I mode collapse occurs when the generative distribution $P_{\hat{\theta}}$ either fails to capture one or more modes of the data-generating distribution $P_X$, or assigns significantly different weights to the modes compared to $P_X$. In heterogeneous populations, failure to capture mixture components or inaccurate allocation of component weights can result in significant misrepresentation of subpopulations within generated samples.
\end{defn}

\begin{defn}[Type II Mode Collapse]
Type II mode collapse occurs when the generative distribution $P_{\hat{\theta}}$ fails to adequately represent the tail behavior of the data-generating distribution $P_X$. 
This results in the potential omission of important extremal information, such as in the case of support shrinkage.
\end{defn}

\section{Theoretical Analysis}\label{Sec:theory}
In this section, we investigate the theoretical properties of the \mpw loss and \potnet. 
In Section \ref{Sect.4.1}, we establish a non-asymptotic generalization bound for the \mpw loss as well as convergence rates of the generative distribution learned by \potnet. Then in Section \ref{Sect.4.2}, we perform a theoretical analysis on how \potnet effectively attenuates mode collapse. 

\subsection{Generalization Bound and Convergence Rates}\label{Sect.4.1}
We revisit the setups outlined in Sections \ref{Sec:prelim} and \ref{Sec:method}. Recall that the observed data $\{X^{(i)}\}_{i=1}^n$ (where $X^{(i)}\in\mathbb{R}^d$) are generated i.i.d. from the true data-generating distribution $P_X(\cdot)$, and that \potnet aims to find the parameter $\hat{\theta}\in\Theta$ that solves the minimization problem \eqref{theta_es}. For any $j\in [d]$ and $\theta\in\Theta$, we denote $P_{X,j}:=(p_j)_{*}P_X$, $P_{\theta,j}:=(p_j)_{*}P_\theta$, and $\PXj:=(p_j)_{*}\PX$.

In Theorem \ref{Theorem1.1} below, we establish a generalization bound for the \mpw loss where we bound the test error $\mathcal{D}(P_X,P_{\theta})$ of \potnet in terms of the training error $\mathcal{D}(\PX,P_{\theta})$.

\begin{thm}\label{Theorem1.1}
Assume that the support of $P_X$ is contained in $B_M=\{\mathbf{x}\in\mathbb{R}^d:\|\mathbf{x}\|_2\leq M\}$, where $M$ is a positive deterministic constant. Then for any $\delta\in (0,1\slash 2)$, with probabiliy at least $1-2\delta$, 
\begin{equation*}
    \mathcal{D}(P_X,P_{\theta})\leq \mathcal{D}(\PX,P_{\theta}) + 2\hat{R}_n(\mathcal{F}_{\pmb{\lambda}})+3M(1+\|\pmb{\lambda}\|_2)\sqrt{\frac{2\log(\delta^{-1})}{n}}, \quad\text{ for all }\theta\in\Theta,
\end{equation*}
where $\pmb{\lambda}$ is defined as in Definition \ref{def:mpw-dist}, $\mathcal{F}_{\pmb{\lambda}}$ is the set of functions $\phi:\mathbb{R}^d\rightarrow\mathbb{R}$ such that $\phi=f+\sum_{j=1}^d\lambda_j f_j\circ p_j$ for some $f\in\mathrm{Lip}_{1,d}$ and $f_1,\cdots,f_d\in\mathrm{Lip}_{1,1}$, and
\begin{equation*}
    \hat{R}_n(\mathcal{F}_{\pmb{\lambda}}):=n^{-1}\mathbb{E}_{\pmb{\epsilon}}\Big[\sup_{\phi\in\mathcal{F}_{\pmb{\lambda}}}\Big\{\sum_{i=1}^n\epsilon_i \big(\phi(X^{(i)})-\phi(0)\big)\Big\}\Big],
\end{equation*}
in which
$\pmb{\epsilon}:=(\epsilon_i)_{i=1}^n$ and $\{\epsilon_i\}_{i=1}^n$ are i.i.d. Rademacher random variables. 
\end{thm}
\begin{rmk}
Note that $\hat{R}_n(\mathcal{F}_{\pmb{\lambda}})$ is the empirical Rademacher complexity of the centered function class $\{\phi(\cdot)-\phi(0):\phi\in \mathcal{F}_{\pmb{\lambda}}\}$. The centering by $\phi(0)$ in the expression of $\hat{R}_n(\mathcal{F}_{\pmb{\lambda}})$ ensures the finiteness of $\hat{R}_n(\mathcal{F}_{\pmb{\lambda}})$. $\mathcal{F}_{\pmb{\lambda}}$ is related to the dual representation of the \mpw distance (see Appendix B).  
\end{rmk}

Next, we establish quantitative convergence rates of the generative distribution $P_{\hat{\theta}}$ learned by \potnet (see \eqref{theta_es}) to the true data-generating distribution $P_X$ in terms of the $1$-Wasserstein distance for both the joint and marginal distributions.

\begin{thm}\label{Theorem1.2}
Assume that the support of $P_X$ is contained in $B_M=\{\mathbf{x}\in\mathbb{R}^d:\|\mathbf{x}\|_2\leq M\}$, where $M$ is a positive deterministic constant. Let $\hat{\theta}$ be defined as in \eqref{theta_es}. Then there exists a positive deterministic constant $C$ that only depends on $M$, such that
\begin{equation}\label{E2.7}
    \mathbb{E}[W_1(P_X,P_{\hat{\theta}})] \leq \mathbb{E}\big[\inf_{\theta\in\Theta}\{\mathcal{D}(\PX,P_{\theta})\}\big]+C\sqrt{d}\cdot 
    \begin{cases}
        n^{-1\slash 2} & \text{ if } d=1\\
        (\log{n}\slash n)^{1\slash 2} & \text{ if }d=2\\
        n^{-1\slash d} & \text{ if }d\geq 3
    \end{cases},
\end{equation}
and for each $j\in [d]$, 
\begin{equation}\label{E2.8}
    \mathbb{E}[ W_1(P_{X,j},P_{\hat{\theta},j})] \leq \lambda_j^{-1}\mathbb{E}\big[\inf_{\theta\in\Theta}\{\mathcal{D}(\PX,P_{\theta})\}\big]+C n^{-1\slash 2}.
\end{equation}
\end{thm}
\begin{rmk}
The term $\inf\limits_{\theta\in\Theta}\{\mathcal{D}(\PX,P_{\theta})\}$ in (\ref{E2.7}) and (\ref{E2.8}) represents the training error which is related to how well the generator network can approximate the empirical measure $\PX$ of observed data. 
Notably, due to the exact optimal transport solver used in \potnet, our upper bounds do not involve an error term arising from the critic network (c.f. \cite[Corollary 16]{liang2021well}).  
\end{rmk}
\begin{rmk}
Note that the term $Cn^{-1\slash 2}$ in the marginal error rate (\ref{E2.8}) does not suffer from the curse of dimensionality. Moreover, as we increase the $j$th marginal penalty (i.e., increasing $\lambda_j$), the contribution of the training error to the $j$th marginal error rate decreases. This demonstrates the marginal regularizing effect of \potnet.
\end{rmk}

\subsection{Theoretical Analysis of Type I and II Mode Collapse}\label{Sect.4.2}

In this subsection, we present a theoretical analysis elucidating how \potnet addresses both types of mode collapse through marginal regularization. 
These theoretical insights are validated by our empirical evaluations in Sections \ref{Sec:sim} and \ref{Sec:app}, which demonstrate that \potnet effectively mitigates mode collapse.

We note that the minimization problem in (\ref{theta_es}) is equivalent to the following constrained minimization problem:
\begin{align}\label{constrained_min}
     \min_{\theta\in\Theta}  W_1(\PX,P_{\theta}),\quad
    \text{ subject to } W_1(\PXj,P_{\theta,j})\leq \epsilon_j \text{ for all }j\in [d],
\end{align}
where $\{\epsilon_j\}_{j=1}^d$ ($\epsilon_j\geq 0$ for each $j\in [d]$) are determined by $\{\lambda_j\}_{j=1}^d$. Throughout the rest of this subsection, we assume that $\hat{\theta}$ is the solution to \eqref{constrained_min}, and that the data-generating distribution $P_X$ admits a density $p_X(\cdot)$ with respect to the Lebesgue measure on $\mathbb{R}^d$.

\paragraph{Analysis of type I mode collapse} 
In the following, we show that the generative distribution $P_{\hat{\theta}}$ learned by \potnet accurately recovers cluster weights in finite mixture models. We consider application to data $\{X^{(i)}\}_{i=1}^n$ generated from a $d$-dimensional mixture model with two components. Specifically, we assume $p_X(\mathbf{x})=\sum_{k=1}^2 w_k \rho(\mathbf{x}-\pmb{\mu}_k)$ for all $\mathbf{x}\in\mathbb{R}^d$, where $\rho(\cdot)$ is a fixed probability density function on $\mathbb{R}^d$, $w_1\in [0,1]$, $w_2=1-w_1$, and $\pmb{\mu}_1,\pmb{\mu}_2\in\mathbb{R}^d$. Without loss of generality, we assume that $\pmb{\mu}_1=(\Delta,\cdots,\Delta)$ and $\pmb{\mu}_2=(-\Delta,\cdots,-\Delta)$, where $\Delta>0$. We note that results similar to Theorem \ref{Thm_GMM} below can be established for mixture models with $K$ components for any finite $K\in\mathbb{N}_{+}$.   

As is commonly assumed in the literature on clustering for finite mixture models (see, e.g., \cite{loffler2021optimality}), we assume that $\pmb{\mu}_1$ and $\pmb{\mu}_2$ are well-separated. Specifically, we assume that $\Delta$ grows to infinity as the sample size $n$ increases. To examine type I mode collapse, for any $k\in [2]$ and $r>0$, we define
\begin{equation*}
\hat{w}_{k,r}:=P_{\hat{\theta}}\big(\big\{\mathbf{x}=(x_1,\cdots,x_d)\in\mathbb{R}^d:\max_{j\in[d]}\{|x_j-\mu_{k,j}|\}\leq r\big\}\big),
\end{equation*}
where $\mu_{k,j}$ is the $j$th coordinate of $\pmb{\mu}_k$. Assuming that $\int_{\mathbb{R}^d\backslash [-r,r]^d}\rho(\mathbf{x})d\mathbf{x}\rightarrow 0$ and $r\slash\Delta\rightarrow 0$ as $n\rightarrow\infty$, for each $k\in[2]$, $\hat{w}_{k,r}$ can be interpreted as the mass assigned by the learned generative distribution $P_{\hat{\theta}}$ to the $k$th mode. We denote by $\hat{\mathcal{R}}$ the minimum value of $W_1(\PX,P_{\theta})$ in the constrained minimization problem \eqref{constrained_min}, which represents the training error of the generator neural network. For any $j\in [d]$, we denote by $\rho_j(\cdot)$ the $j$th marginal density of $\rho(\cdot)$. We assume that for all $j\in [d]$ and $r\geq 0$, $\sup\limits_{x\in\mathbb{R}}\rho_j(x)\leq M$ and $\rho_j([-r,r]^c)\leq \psi(r)$, where $M\geq 0$ and $\psi(\cdot)$ is a function that maps $[0,\infty)$ into $[0,\infty)$.

The following result provides a non-asymptotic upper bound on the discrepancy between $\hat{w}_{k,r}$ and $w_k$ for all $k\in [2]$ and $r\in (0,\Delta)$.  

\begin{thm}\label{Thm_GMM}
Assume the preceding setups. Then for any $\delta\in (0,1)$, with probability at least $1-\delta$, for all $k\in [2]$ and $r\in (0,\Delta)$, we have
\begin{equation}
     \big|\hat{w}_{k,r}-w_k\big|\leq 2d\psi(r)+8d\sqrt{M\max_{j\in[d]}\{\epsilon_j\}}+5d\sqrt{\frac{\log(16 d\slash \delta)}{2n}}+\frac{\hat{\mathcal{R}}}{ 2(\Delta-r)}. 
\end{equation}
\end{thm}

Theorem \ref{Thm_GMM} leads to
the following corollary on two-component Gaussian mixture models.

\begin{cor}\label{Cor1.1}
Let $\rho(\cdot)$ be the density of $\mathcal{N}(0,\sigma^2\mathbf{I}_d)$, where $\sigma>0$. Then for any $\delta\in (0,1)$, with probability at least $1-\delta$, for all $k\in [2]$ and $r\in (0,\Delta)$, we have
\begin{align}\label{E4.25}
    \big|\hat{w}_{k,r}-w_{k}\big|\leq4d\exp\Big(-\frac{r^2}{2\sigma^2}\Big)+8 d\sqrt{\frac{\max_{j\in[d]}\{\epsilon_j\}}{\sigma}}+5d\sqrt{\frac{\log(16d\slash \delta)}{2n}}+\frac{\hat{\mathcal{R}}}{ 2(\Delta-r)}.
\end{align}
\end{cor}
\begin{rmk}
Assuming that $\sigma\sqrt{\log{d}}\ll r\ll \Delta$ as $n\rightarrow\infty$, for each $k\in[2]$, $\hat{w}_{k,r}$ can be interpreted as the mass assigned by $P_{\hat{\theta}}$ to the $k$th mode. Further assuming that $\max_{j\in[d]}\{\epsilon_j\}\ll \sigma d^{-2}$, $n\gg d^2\log{d}$, and $\hat{\mathcal{R}}\ll \Delta$ (note that these assumptions apply to a broad range of practical applications), it follows from \eqref{E4.25} that $\hat{w}_{k,r}$ gives a consistent approximation to $w_k$. This shows that the learned generative distribution $P_{\hat{\theta}}$ effectively recovers the correct weight of each mode and avoids type I mode collapse. The $n^{-1\slash 2}$ dependence on the sample size $n$ in the error term $5d\sqrt{\log(16d\slash \delta)\slash(2n)}$ in \eqref{E4.25} arises from the fast convergence rate of the marginal distribution $\PXj$ to $P_{X,j}$ for all $j\in [d]$.
\end{rmk}

\paragraph{Analysis of type II mode collapse}

In the following, we show that the generative distribution $P_{\hat{\theta}}$ learned by \potnet retains the tail decay exponent of the data-generating distribution $P_X$ under mild conditions, i.e, it can effectively mitigate type II mode collapse.

\begin{thm}\label{Thm_T}
Assume that $P_X(\{\mathbf{x}\in\mathbb{R}^d:\|\mathbf{x}\|_2\geq r\})\geq c_0\exp(-C_0 r^{\alpha})$ for sufficiently large $r$, where $c_0,C_0,\alpha$ are positive constants that only depend on $p_X(\cdot),d$. Then for any $\delta\in (0,1)$, with probability at least $1-\delta$, we have
\begin{equation}\label{E4.12}
     P_{\hat{\theta}}(\{\mathbf{x}\in\mathbb{R}^d:\|\mathbf{x}\|_2\geq r\})\geq c\exp(-C r^{\alpha})-\sqrt{\frac{2\log(4d\slash \delta)}{n}}-2\max_{j\in[d]}\{\epsilon_j\}
\end{equation}
for all sufficiently large $r$, where $c,C$ are positive constants that only depend on $p_X(\cdot),d$. 
\end{thm}
\begin{rmk}
Assuming that $n\gg \log{d}$ and $\max_{j\in[d]}\{\epsilon_j\}\rightarrow 0$ as $n\rightarrow\infty$, it follows from \eqref{E4.12} that the generative distribution $P_{\hat{\theta}}$ learned by \potnet effectively preserves the tail decay exponent of the data-generating distribution $P_X$. The $n^{-1\slash 2}$ dependence on the sample size $n$ in the error term $\sqrt{2\log(4d\slash \delta)\slash n}$ in \eqref{E4.12} arises from the fast convergence rate of the marginal distribution $\PXj$ to $P_{X,j}$ for all $j\in [d]$. 
\end{rmk}

\paragraph{Marginal regularization attenuates mode collapse}
Here we summarize the theoretical intuitions that underpin the desirable performance of \potnet provided by Theorems \ref{Thm_GMM}-\ref{Thm_T} and Corollary \ref{Cor1.1}.
\textit{For type I mode collapse}: the marginal penalty enforces stringent constraints on the possible mass allocations by the generative distribution to different regions of $\mathbb{R}^d$. Generative distributions that satisfy these constraints but assign inaccurate mode weights would necessarily have a large $W_1$ distance to the true data-generating distribution (see the proof for Theorem \ref{Thm_GMM} presented in Appendix D). 
Therefore, generative distributions suffering from type I mode collapse will not be selected by \potnet and consequently, the learned generative distribution assigns accurate weights to each mode. 
\textit{For type II mode collapse}: the fast convergence rate of marginal distributions ensures that empirical marginal distributions preserve the tail information of the true marginal distributions well. The marginal penalty thus forces the learned generative distribution to align well with the target data-generating distribution even in low-density regions, ensuring the preservation of accurate tail characteristics.

\section{Empirical Evaluations}\label{Sec:sim}
In this section, we provide empirical evidence demonstrating the efficacy of \potnet in comparison with four other methods. 
Our analysis includes the state-of-the-art generative model for tabular data, CTGAN\footnote{We utilized the implementation provided at \url{https://github.com/sdv-dev/CTGAN}} \citep{xu2019modeling}, and WGAN with Gradient Penalty \citep{gulrajani2017improved}.
To underscore the effectiveness of the \mpw loss, we further evaluate \potnet against variants of the same generative model---utilizing an identical generator network---but optimized using alternative loss functions, including the unpenalized 1-Wasserstein distance (denoted as OT) and the Sliced Wasserstein distance (denoted as SW).
In the following, we present results from four distinct simulation experiments with dataset-specific benchmarking criteria.
All methods adopt identical network structures, learning rates, and other relevant hyperparameters (such as the latent dimension, number of epochs) where applicable.

\paragraph{Bayesian Likelihood-free inference}
Simulator models, models which generate samples easily but lack explicit likelihoods, are widely used for stochastic system modeling in science and engineering.
However, conducting parametric statistical inference on observed data from such models often proves challenging due to the unavailability or intractability of likelihood computations \citep{papamakarios2019sequential}.
Approximate Bayesian Computation (ABC) tackles this challenge by utilizing likelihood-free rejection sampling to identify parameters where the simulated data is close to the real observation. 
An important step in many ABC methods involves generating new samples from approximate posterior distributions \citep{wang2022adversarial}.

We consider a widely-used benchmark model  in likelihood-free inference \citep{papamakarios2019sequential, wang2022adversarial}.
In this setup, $\pmb{\phi}$ is a 5-dimensional vector. 
For each $\pmb{\phi}$, we observe four sets of bivariate Gaussian samples $\pmb{x} \in \R^8$, whose mean and covariance depend on $\pmb{\phi}$ as follows:
\begin{align*}
    \phi_i &\sim \mathrm{Unif}[-3, 3] \quad \text{for } i = 1, \dots, 5, \\
    x_j &\sim \mathcal{N}(\mu_\phi, \Sigma_\phi) \quad \text{for } j \in [4], \\
    \text{where} \quad \mu_\phi &= 
    \begin{pmatrix}
        \phi_1 \\ 
        \phi_2
    \end{pmatrix}, \quad
    s_1 = \phi_3^2, \quad s_2 = \phi_4^2, \quad \rho = \tanh(\phi_5), \\
    \Sigma_\phi &= 
    \begin{pmatrix}
        s_1^2 & \rho s_1 s_2 \\
        \rho s_1 s_2 & s_2^2
    \end{pmatrix}.
\end{align*}
The posterior distribution of this simple model has truncated support and features four unidentifiable modes since the signs of $\phi_3$ and $\phi_4$ are indistinguishable. 
The ground truth dataset is obtained via rejection sampling where a parameter is retained if the corresponding $\pmb{X}^{(i)}$ falls within an $\epsilon$-ball of the observed data $\pmb{X}_{\mathrm{obs}} \in \R^8$. Formally, our dataset is $\{\pmb{\phi}^{(i)} \,|\, \|\pmb{X}^{(i)} - \pmb{X}_\mathrm{obs}\|_1 < \epsilon\}_{i=1}^{3000}$ with $\epsilon = 1.5$.

\begin{figure}[htbp]
    \centering
    \subfloat[Marginal distributions.]{{
        \includegraphics[width=.6\textwidth]{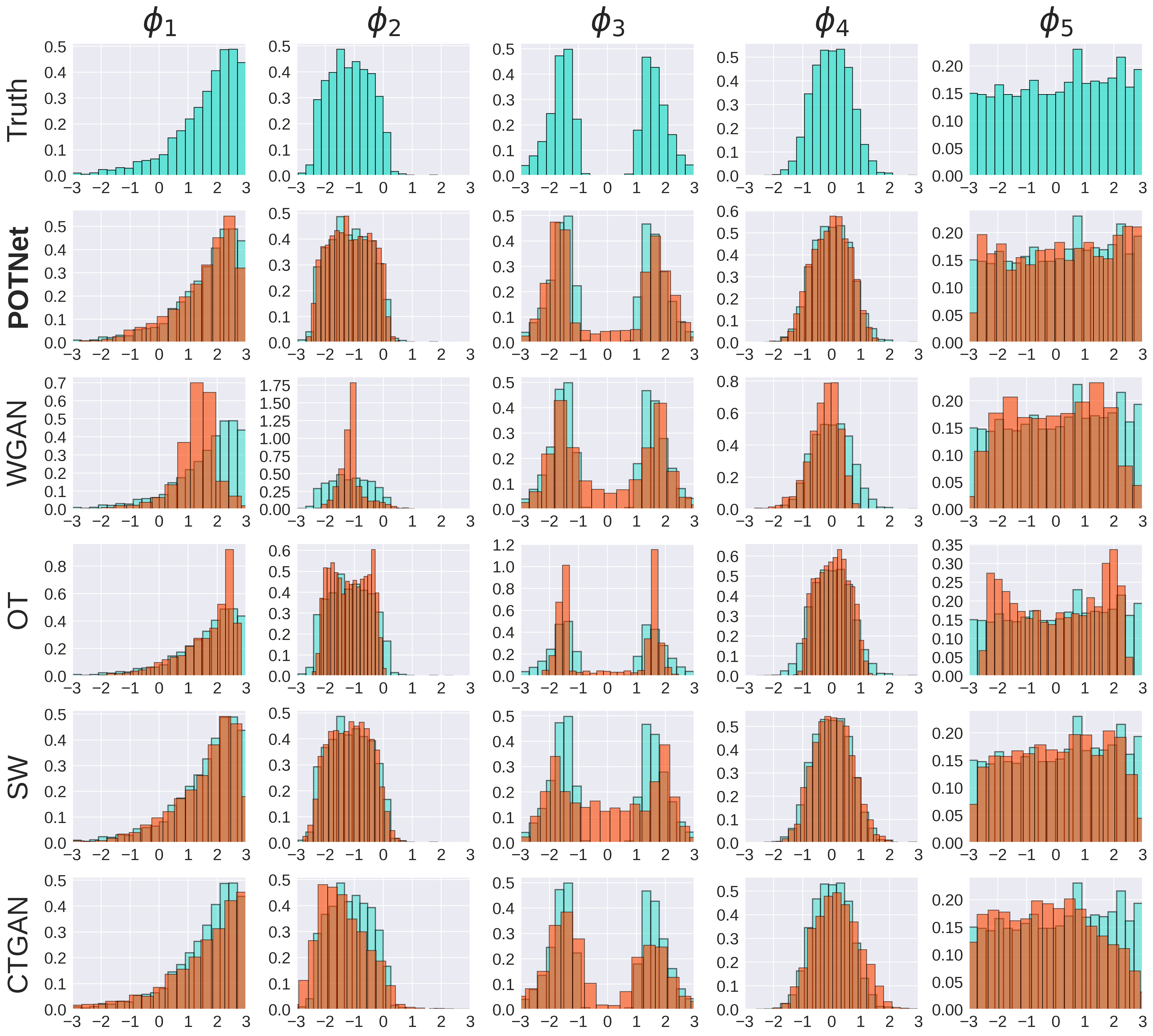}\label{fig:abc-marginal_distributions}
    }} \\
    \subfloat[Bivariate contours.]{{
        \includegraphics[width=.6\textwidth]{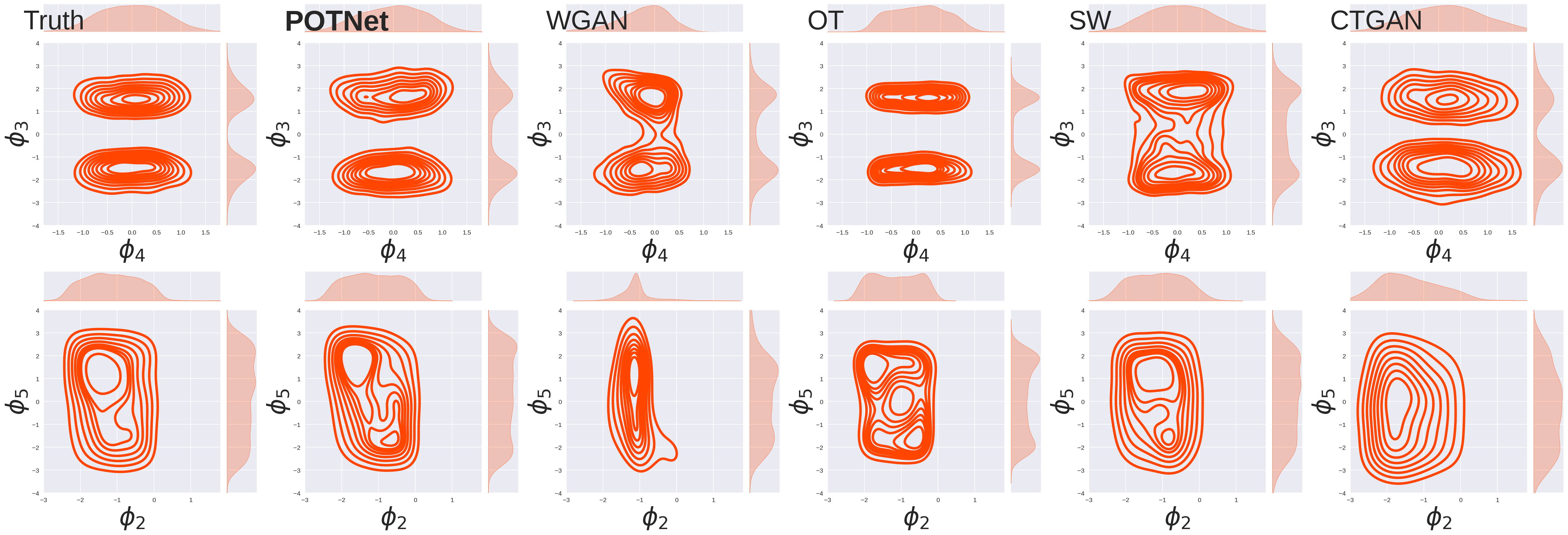}\label{fig:abc-bivariate}
    }} \\
    \subfloat[Deviation of sample covariance matrix.]{{
        \includegraphics[width=.6\textwidth]{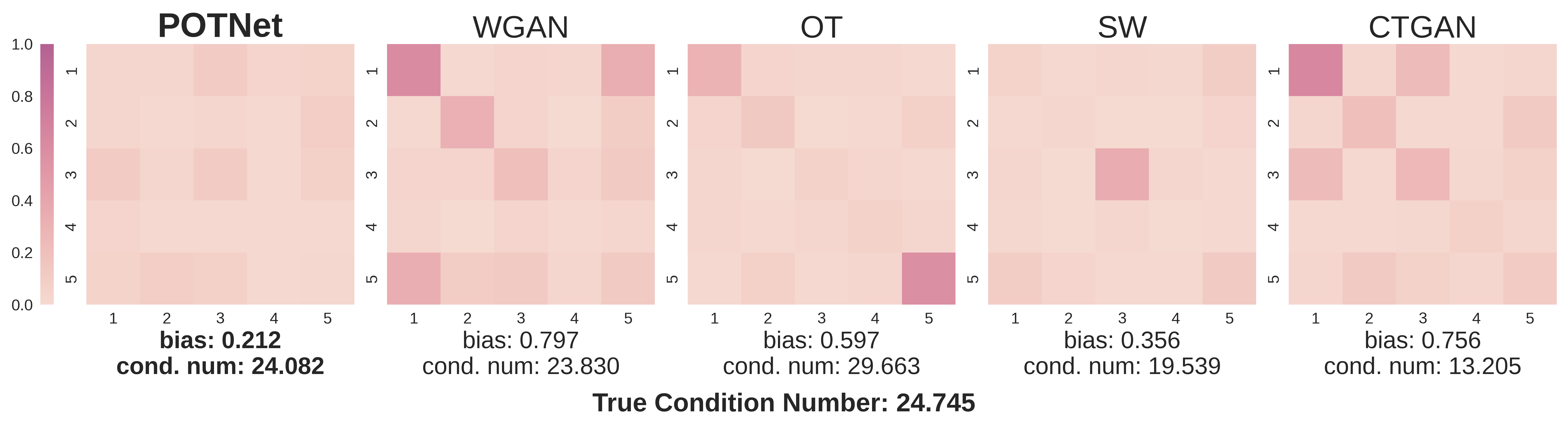}\label{fig:covar-mat}
    }}
    \caption{Performance of five methods for estimating the approximate posterior distribution of $\pmb{\phi}$. 
    \textbf{(a)} Marginal distributions for each $\phi_j$, $j=1,\dots,5$, by column. Cyan: ground truth; orange: synthetic samples. 
    \textbf{(b)} Bivariate density contour plots: $\phi_4$ vs $\phi_3$ (top panel) and $\phi_2$ vs $\phi_5$ (bottom panel). 
    \textbf{(c)} Heatmap of absolute deviation of sample covariance matrix from estimated covariance matrix of the ground truth dataset (lighter color indicates smaller deviation). Bias: spectral norm of difference between synthetic and real covariance (lower is better). Condition number: $\lambda_{\mathrm{max}} / \lambda_{\mathrm{min}}$ of covariance matrix.}
    \label{fig:abc-comparison}
\end{figure}

The comparative analysis of various generative models is presented in Figure \ref{fig:abc-comparison}. In plot (a), WGAN and OT exhibit severe mode collapse for $\phi_2$ and $\phi_3$, respectively, while CTGAN distorts the distribution of $\phi_5$. In contrast, \potnet effectively learns the distribution without experiencing mode collapse for any parameter. Plot (b) further reveals that OT exhibits significant type II mode collapse, evidenced by vertically compressed support in bivariate distributions. WGAN and SW fail to capture the symmetric bimodality of $\phi_3$, whereas CTGAN incorrectly learns the mode of $\phi_5$. Plot (c) demonstrates the superior performance and robustness of \potnet in preserving joint relationships, achieving the smallest bias, which is defined as the spectral norm of deviation between synthetic and real data sample covariance matrices. Moreover, the estimated condition number of the sample covariance matrix from \potnet (24.1) closely approximates that of the real data (24.7), where the condition number is defined as $\lambda_{\mathrm{max}} / \lambda_{\mathrm{min}}$ with $\lambda_{\boldsymbol{\cdot}}$ denoting the eigenvalues of the sample covariance matrix. 
We remark that \potnet demonstrates robust performance across random initializations and exhibits rapid convergence in practice. 
For more details, including a three-dimensional comparison of the generated samples, please refer to Appendix E.

Next, we evaluate the quality of generated data using two dissimilarity-type metrics, \emph{Maximum Mean Discrepancy} (MMD) with Gaussian kernel and \emph{Total Variation Distance} (TV dist.) summarized in Table \ref{tab:abc-comparison}.
MMD quantifies the maximum deviation in the expectation of the kernel function evaluated on samples. Total variation distance is an \textit{f}-divergence that measures the maximal difference between assignments of probabilities of two distributions (see Appendix E for details on empirical estimation of TV distance via discretization).
\potnet achieves highest performance in synthetic data quality under both metrics.

\begin{table}[t!]
\begin{center}
\begin{sc}
\begin{small}
    \begin{tabular}{l|ccccc}
    \toprule
    Metric & \textbf{\potnet} & WGAN & OT & SW & CTGAN \\
    \midrule
    TV dist.        & \textbf{0.560}   & 0.775         & 0.648       & 0.663       & 0.634         \\
    MMD (log)       & \textbf{-6.366}  & -3.474        & -5.587      & -5.900      & -4.953        \\
    \bottomrule
\end{tabular}
\end{small}
\end{sc}
\end{center}
\caption{Evaluation on the ABC dataset using dissimilarity measures: Total Variation (TV) distance and log Maximum Mean Discrepancy (MMD). Lower value indicates better performance. Boldfaced value identifies the best performance in each category.}
\label{tab:abc-comparison}
\end{table}

Additionally, we assess the computational performance of each model in terms of training and sampling time averaged over seven trials. All models are implemented using \texttt{PyTorch} \citep{paszke2019pytorch} and executed on a Tesla T4 GPU. 
The models are trained to convergence, with WGAN requiring 2,000 epochs and the remaining models 200 epochs. For each model, we generate 3,000 samples.
The results are summarized in Table \ref{tab:runtime-comparison}. \potnet requires significantly shorter time for training and sampling compared to CTGAN and WGAN (achieving approximately an 80-fold speedup in the sampling stage compared to CTGAN and a 3-fold speedup compared to WGAN). 
This efficiency makes \potnet particularly well-suited for generating large volumes of synthetic samples, a common need in practical applications \citep{dahmen2019synsys}.

\begin{table}[h!]
\vspace{1em}
\begin{center}
\begin{sc}
\begin{small}
    \begin{tabular}{l|ccccc}
    \toprule
    Runtime& \textbf{\potnet} & WGAN & OT & SW & CTGAN \\
    \midrule
    Train (s)      & 38.1     & 235      & 33.5     & \textbf{7.75} & 56     \\
    Sample ($\mu$s) & \textbf{667} & 1.8$\times 10^3$ & 853      & 759      & 54$\times 10^3$ \\
    \bottomrule
\end{tabular}
\end{small}
\end{sc}
\end{center}
\caption{Runtime comparison of training and sampling times across different methods. \potnet achieves approximately an $80$x speedup in sample generation compared to CTGAN.}
\label{tab:runtime-comparison}
\end{table}

\paragraph{Mixture of Gaussians}
In this example, we investigate the effectiveness of \potnet in mitigating both type I and type II mode collapse.
To do so, we generate a mixture of three 20-dimensional Gaussians with unbalanced weights, creating one large cluster and two small clusters of weights $0.8, 0.1, 0.1$, respectively. Detailed specifications of the true distribution and training procedures are provided in Appendix E.

Figure \ref{fig:gmm-bivariate_contours} displays bivariate contour plots of the first two dimensions for both real and synthetic data, with contours estimated using kernel density estimation.
SW and CTGAN fail to produce high-quality samples: SW incorrectly estimates the covariance of the largest cluster, while CTGAN fails to capture the overall joint relationship.
Notably, both WGAN and OT exhibit type I mode collapse in which they fail to allocate sufficient weight to the smaller clusters. 
OT additionally demonstrates severe support shrinkage and significantly reduces sample diversity. 
This is evident in both the tails (black contour) and the mode, particularly in the largest cluster (purple contour), where the mode displays horizontal compression.
In Appendix E, we demonstrate that the mode collapsing behavior of synthetic data generated by OT leads to catastrophic consequences for downstream retraining tasks.

\begin{figure}[htbp]
\vspace{1em}
\begin{center}
    \includegraphics[width=1\textwidth]{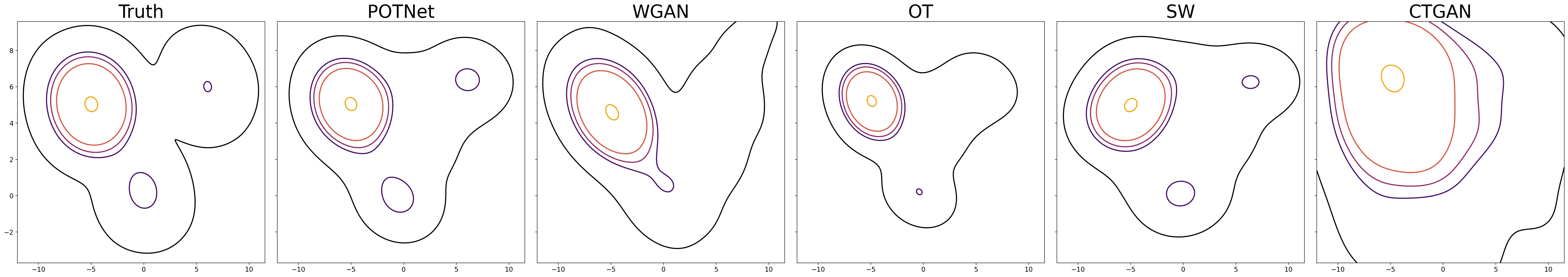}
\caption{Bivariate contour plots of the first two dimensions of a 20D Gaussian mixture model with 3 components. WGAN shows type I (mode dampening) mode collapse while OT exhibit both type I and type II (support shrinkage) mode collapse.}
\label{fig:gmm-bivariate_contours}
\end{center}
\end{figure}

\paragraph{2D manifold embedded in 3D space}
We now evaluate the ability of \potnet to capture complex underlying manifold structure. 
For this assessment, we use the classic S-curve dataset \citep{ma2012manifold}, a standard benchmark for comparing nonlinear dimensionality reduction and manifold learning methods. Each data point $\pmb{X} = (X_1, X_2, X_3)$ is generated by
\begin{align*}
U \sim \mathrm{Unif}(-3\pi/2, 3\pi/2),\quad
    X_1 = \sin(U), \quad X_2 \sim &\mathrm{Unif}(0, 2), \quad X_3 = \mathrm{sgn}(U) (\cos(U)-1)
\end{align*}
Although the ambient dimension (in this case, three) exceeds the manifold dimension (which is two), we set the latent dimension to three for all methods. 
Each method is trained for 1,000 epochs, with the exception of WGAN, which required 2,000 epochs to converge.
Figure \ref{fig:s-curve} presents two-dimensional (top panel) and three-dimensional (bottom panel) visualizations of the synthetic data generated by each model. The plot clearly shows the superior quality of samples generated by \potnet compared to other methods. 
WGAN exhibits support shrinkage with generated samples skewed towards the right. 
Both SW and CTGAN fail to accurately capture the underlying data structure. 
OT demonstrates acceptable performance, although it produces an excess of samples along the left and right edges that are not present in the real data.

\begin{figure}[htbp]
\vspace{1em}
\begin{center}
    \includegraphics[width=1\textwidth]{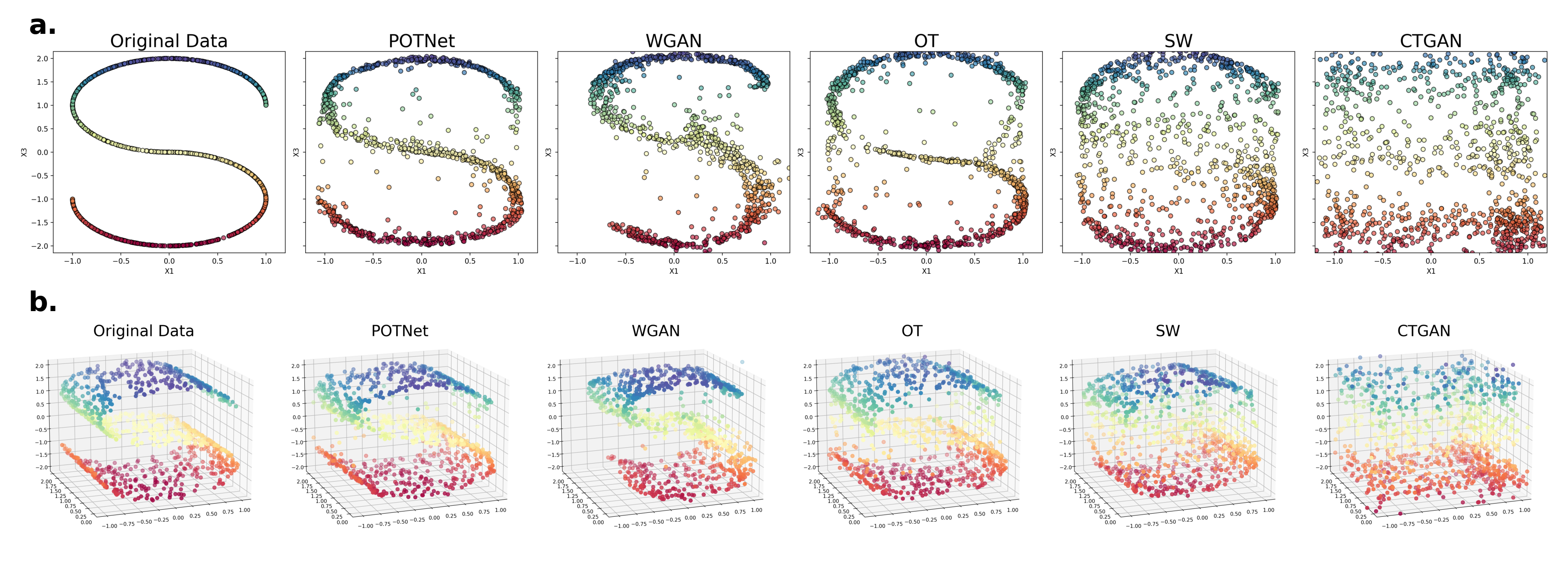}
\caption{Comparison of synthetic data generated by each method in 2D (top) and 3D (bottom) when the underlying manifold is complex and lower-dimensional. }
\label{fig:s-curve}
\end{center}
\end{figure}

\paragraph{Assessment of estimator performance using synthetic data}
In our final simulation, we assess the utility of synthetic data in evaluating estimator performance. 
This setting corresponds to the use case described in \cite{athey2024using}, where a practitioner is interested in using synthetic data to gauge the empirical performance of estimators, such as the coverage rate of confidence intervals. 
As \cite{athey2024using} states, for this purpose, it is important that the generated data \textit{accurately} resembles the true data generating mechanism. 
We generate a real dataset of 300 observations from a 3-dimensional Gaussian distribution $\N(0, \mathbf{\Sigma})$ where 
\begin{align*}
    \mathbf{\Sigma} &= 
    \begin{bmatrix}
    30 & 10.3 & -20.2 \\
    10.3 & 20 & 0.3 \\
    -20.2 & 0.3 & 20
    \end{bmatrix} \in \mathbb{R}^{3 \times 3}
\end{align*} 
Each observation consists of three features, i.e., $\pmb{X} = (X_1, X_2, X_3)$.

We use this dataset to train five models, each of which then generates 1,000 synthetic datasets of 300 samples each. Our analyses focus on inference with respect to the first variable $X_1$. For each synthetic dataset $\{\widetilde{\pmb{X}}^{(b, i)}\}_{i=1}^{300}, b\in [1000]$, we compute the sample mean $\bar{X}_{1,(b)}^*$ and the sample standard deviation $\widehat{\sigma}_{1,(b)}^*$ of the first coordinate, resulting in a vector of 1,000 sample means $\{\bar{X}_{1,(b)}^*\}_{b=1}^{1000}$, and 1,000 confidence intervals. We evaluate:
\begin{enumerate}
    \item[\textbf{S1:}] The bootstrap confidence interval for $\mu_1$, computed via the percentile bootstrap method using the empirical distribution of sample means $\{\bar{X}_{1,(b)}^*\}_{b=1}^{1000}$ from the synthetic datasets. This interval is given by $[q_{0.025}, q_{0.975}]$, where $q_{\alpha}$ denotes the $\alpha$-quantile of the empirical distribution. 
    \item[\textbf{S2:}] The coverage rate of 95\%-confidence intervals for $\mu_1$ computed using each synthetic dataset. Each interval is given by $\{\bar{X}_{1,(b)}^* \pm t_{\alpha/2, n-1} \widehat{\sigma}_{1,(b)}^*/\sqrt{300}\}$ for $\alpha = 0.05$, where $\widehat{\sigma}_{1,(b)}^*$ is the sample estimate of standard deviation for the $b$th synthetic dataset.
\end{enumerate}
Given that the true distribution of $X_1$ is $\mathcal{N}(0, \Sigma_{1,1}) = \mathcal{N}(0, 30)$, we can assess the quality of our synthetic data generation process by comparing the results to this known distribution. 
Specifically, if the synthetic datasets accurately represent samples from the true distribution, 
the percentile bootstrap should deliver an interval close to the true interval of $(\pm z_{\alpha/2} \sqrt{30/300})$, and the sample-based confidence intervals derived from these datasets should exhibit approximately the correct nominal coverage rate of 95$\%$ for the true parameter $\mu_1 = 0$.
This comparison provides a measure of how well the synthetic data capture statistical properties of the original distribution. The results are included in Table \ref{tab:confidence_intervals}.
In setting S1, \potnet produces confidence intervals whose lower and upper bounds deviate the least from those of the true confidence interval.
Additionally, \potnet generates the most symmetric bootstrapped-intervals around the true parameter $\mu_1 = 0$. In contrast, the intervals provided by all other methods exhibit some degree of skewness, with the interval generated by CTGAN notably failing to cover the true parameter value of 0 entirely.
In S2, \potnet-generated synthetic data yields 95\%-confidence intervals with a coverage rate of nearly 95\%, accurately capturing the true parameter at the nominal level. 
Such performance suggests that the generative distribution of \potnet closely approximates the statistical properties of the true distribution, thereby allowing for reliable inference.

\begin{table}[t!]
\begin{center}
\begin{small}
\begin{sc}
\begin{tabular}{l|cc|c}
\toprule
Method & Lower Bound (S1) & Upper Bound (S1) & Coverage (S2) \\
\midrule
\textcolor{blue}{\textbf{Truth}} & \textcolor{blue}{\textbf{-0.62}} & \textcolor{blue}{\textbf{0.62}} & \textcolor{blue}{\textbf{95.0\%}} \\
\textbf{\potnet} & \textbf{-0.50} & \textbf{0.63} & \textbf{93.8\%} \\
WGAN & -0.18 & 0.99 & 77.4\% \\
OT & -0.29 & 0.85 & 82.8\% \\
SW & -0.84 & 0.39 & 87.5\% \\
CTGAN & 0.37 & 1.03 & 73.6\% \\
\bottomrule
\end{tabular}
\end{sc}
\end{small}
\end{center}
\caption{Comparison of confidence intervals and coverage rates across five synthetic data generation methods. Each dataset is of size $n=300$.}
\label{tab:confidence_intervals}
\end{table}

\section{Applications to Real Data}\label{Sec:app}
In this section, we evaluate the generative capability of the proposed method on four real-world datasets from diverse domains to demonstrate the wide applicability of \potnet and the \mpw loss. These datasets include three tabular datasets and one image dataset.
The tabular datasets are: (1) The \emph{California Housing} dataset, a widely used benchmark for regression analysis \citep{pace1997sparse}, and two classification benchmarks from the UCI Machine Learning Repository: (2) the \emph{Breast Cancer} diagnostic data, comprising 9 categorical features and one discrete feature \citep{misc_breast_cancer_14}, and (3) the \emph{Heart Disease} dataset, consisting of 8 categorical and 6 numeric features \citep{misc_heart_disease_45}. 
To further showcase that the \mpw loss is not limited to tabular data modeling, we apply it to digit generation using the MNIST handwritten digits dataset, a standard benchmark widely used for both digit recognition and generative tasks in computer vision.
As first three tabular datasets comprise a mix of continuous, discrete, and categorical features, WGAN cannot be easily adapted to model these datasets with feature heterogeneity; therefore, we omit it from the experiments involving these tabular data.

\subsection{Assessment of Tabular Data Generation}
To comprehensively assess the quality of the generated tabular data distributions, we will compare the performance of \potnet, OT, SW, and CTGAN using two evaluation criteria: \emph{machine learning efficacy} and bivariate joint plot inspection.

\paragraph{Machine learning efficacy}
An ubiquitous metric for examining generative modeling on real datasets is the machine learning efficacy, which measures the effectiveness of synthetic data in performing prediction tasks \citep{xu2019modeling}. 
Our evaluation process consists of four steps:
\begin{enumerate}
    \item Split the real data into training and test sets.
    \item Train generative models on the real training set to produce synthetic datasets.
    \item Train classifiers or regressors on synthetic training sets.
    \item Assess the performance of classifier/regressor on the real test set.
\end{enumerate}
To reduce the influence of any particular discriminative model on our evaluation, we employ both Random Forest (RF) and Decision Tree (DT) methods in our evaluation.
Table \ref{tab:real-tab-data} summarizes the results from five repetitions of the experiments. 
For the regression task on the \emph{California Housing} dataset, we evaluate performance using mean squared error (MSE), with lower values indicating better performance. For the \emph{Breast} and \emph{Heart} datasets, performance is assessed using classification accuracy, where higher values are better. The results demonstrate that \potnet consistently outperforms all competing methods on all three datasets, each characterized by distinct structures and feature compositions.

\begin{table}[htbp]
\vspace{1em}
\begin{center}
\begin{small}
\begin{sc}
\resizebox{.8\textwidth}{!}{%
\begin{tabular}{lcc|cccc}
\toprule
Data  & Metric &  & \potnet & OT & SW & CTGAN \\
\midrule
\multirow{4}{*}{CA housing} & \multirow{2}{*}{DT (MSE) $\downarrow$} & mean & \textbf{0.155} & 0.174 & 0.222 & 0.307 \\
& & std & 0.0088 & 0.0126 & 0.0501 & 0.0962 \\ 
& \multirow{2}{*}{RF (MSE) $\downarrow$} & mean & \textbf{0.107} & 0.192 & 0.169 & 0.462 \\
& & std & 0.0035 & 0.0437 & 0.0292 & 0.0869 \\
\midrule 
\multirow{4}{*}{Breast} 
& \multirow{2}{*}{DT $\uparrow$} & mean & \textbf{0.696} & 0.644 & 0.650 & 0.627 \\
& & std & 0.0382 & 0.0363 & 0.0500 & 0.0478 \\ 
& \multirow{2}{*}{RF $\uparrow$} & mean & \textbf{0.707} & 0.664 & 0.683 & 0.671 \\
& & std & 0.0273 & 0.0316 & 0.0294 & 0.0407 \\ 
\midrule 
\multirow{4}{*}{Heart} &  \multirow{2}{*}{DT $\uparrow$} & mean & \textbf{0.653} & 0.633 & 0.649 & 0.525 \\
&& std & 0.0339 & 0.0371 & 0.0338 & 0.1014 \\
& \multirow{2}{*}{RF $\uparrow$} & mean & \textbf{0.699} & 0.686 & 0.673 & 0.611 \\
& & std & 0.0264 & 0.0170 & 0.0182 & 0.0596 \\
\bottomrule
\end{tabular}
}
\caption{Machine learning efficacies (first row: mean squared error (MSE), \emph{lower} values indicate better performance; last two rows: accuracy, \emph{higher} values indicate better performance) on real datasets computed over 5 repetitions.}
\label{tab:real-tab-data}
\end{sc}
\end{small}
\end{center}
\end{table}

\begin{figure}[h!]
\begin{center}
    \includegraphics[width=.9\textwidth]{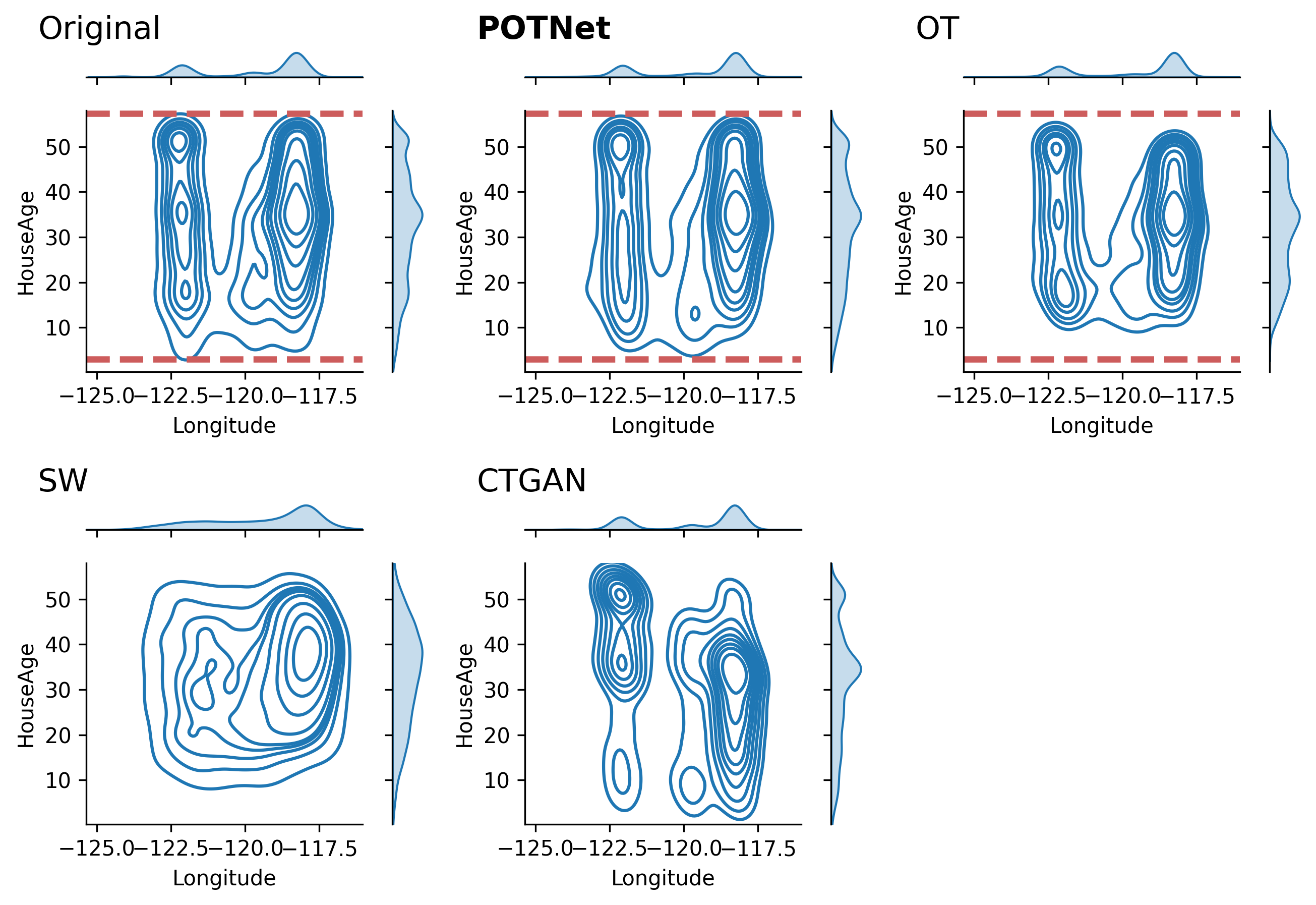}
 \caption{A comparison of the bivariate joint distribution of \textit{Longitude} and \textit{HouseAge} for the California Housing dataset. Red lines: upper and lower bounds for contour lines of the original data. OT fails to adequately capture the variance of \textit{HouseAge}.}
\label{fig:ca-housing}
\end{center}
\end{figure}

Next, we conduct a visual assessment of the synthetic data generated by each method through examination of bivariate distributions.
Figure \ref{fig:ca-housing} presents joint density plots for two multimodal features in the \emph{California Housing} dataset: \emph{Longitude} and \emph{HouseAge}.
The plot reveals that SW fails to capture modal separation, likely due to the rotational invariance of its random projections. CTGAN misrepresents the modal spread and positioning. OT exhibits severe mode collapse, failing to fully capture the data support with significant shrinkage in the distribution tails.
In contrast, the axis-dependent nature and marginal regularization of \potnet yield high-fidelity synthetic data that faithfully mimics the target distribution.

\subsection{Comparative Evaluation of Image Generation}
To demonstrate the versatility of \potnet beyond tabular data modeling, we apply it to data with grid-like structure, specifically images. 
For this purpose, we focus on image generation using the MNIST dataset, a collection of 70,000 grayscale images (28x28 pixels) depicting handwritten digits from 0 to 9.
\paragraph{Setting 1} We evaluate the performance of \potnet, OT, and WGAN using the same conventional convolutional architecture.
Figure \ref{fig:mnist-convolution} displays the digits generated after 100 epochs of training.
The results demonstrate that \potnet produces the most consistently clear and recognizable handwritten digits.
OT produces slightly less defined results (e.g., first digit of third row), while WGAN renders output with greater variability and reduced clarity in some digit formations (e.g., the first digit of the first row appears to be a combination of ``4'' and ``9'').

\begin{figure}[h!]
\begin{center}
    \includegraphics[width=1.0\columnwidth]{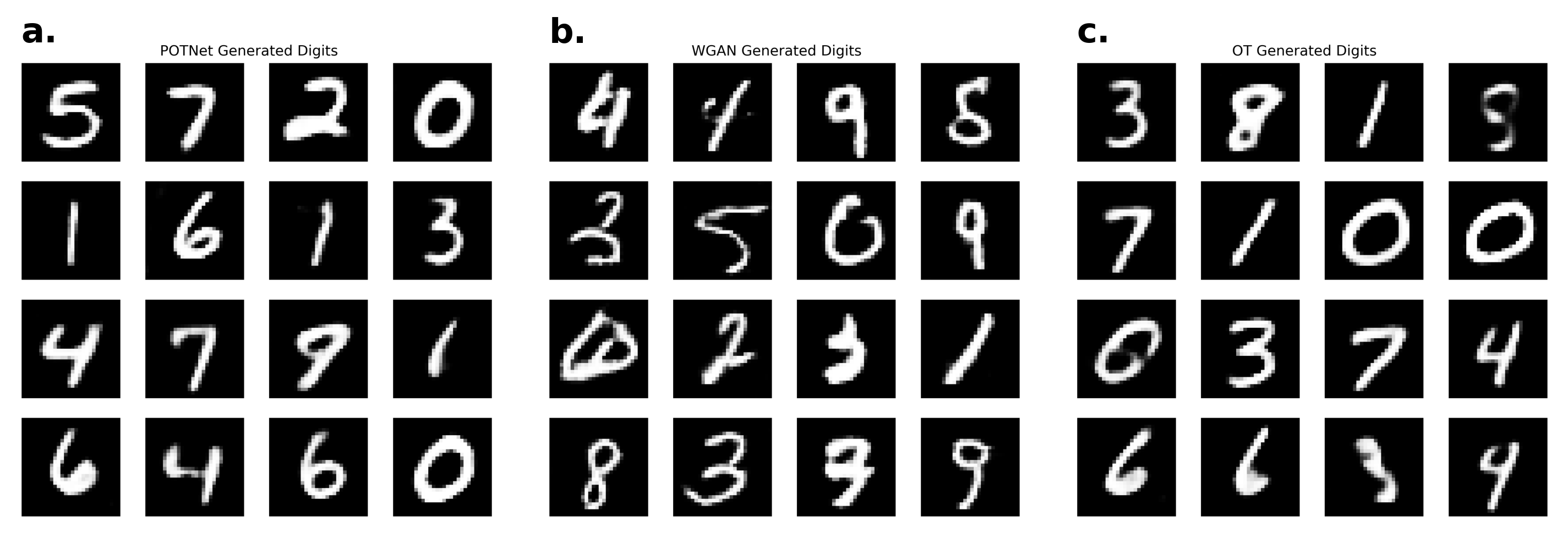}
\caption{Comparison of generated MNIST handwritten digits after 100 epochs of training using a convolutional network architecture. (a). \potnet; (b). WGAN; (c). OT.}
\label{fig:mnist-convolution}
\end{center}
\end{figure}

\paragraph{Setting 2} 
Now we demonstrate the ability of \potnet to capture the underlying manifold structure without relying on convolutional layers, and the pronounced mode collapse behavior manifested by OT. 
To illustrate these points, we employ a four-layered feedforward network for both methods instead of using convolutional layers.
Each model is trained for 300 epochs to convergence. The resulting generated digits are presented in Figure \ref{fig:mnist-ffnn}.
Overall, \potnet demonstrates robust performance and 
produces well-defined and easily identifiable digits. 
In contrast, OT generates digits that exhibit significant degradation, resulting in many indistinct or incomplete forms. The results show that OT suffers from pronounced mode collapse and fails to capture the full range of pixel values present in the MNIST dataset.
\begin{figure}[htbp]
\begin{center}
    \includegraphics[width=1.0\columnwidth]{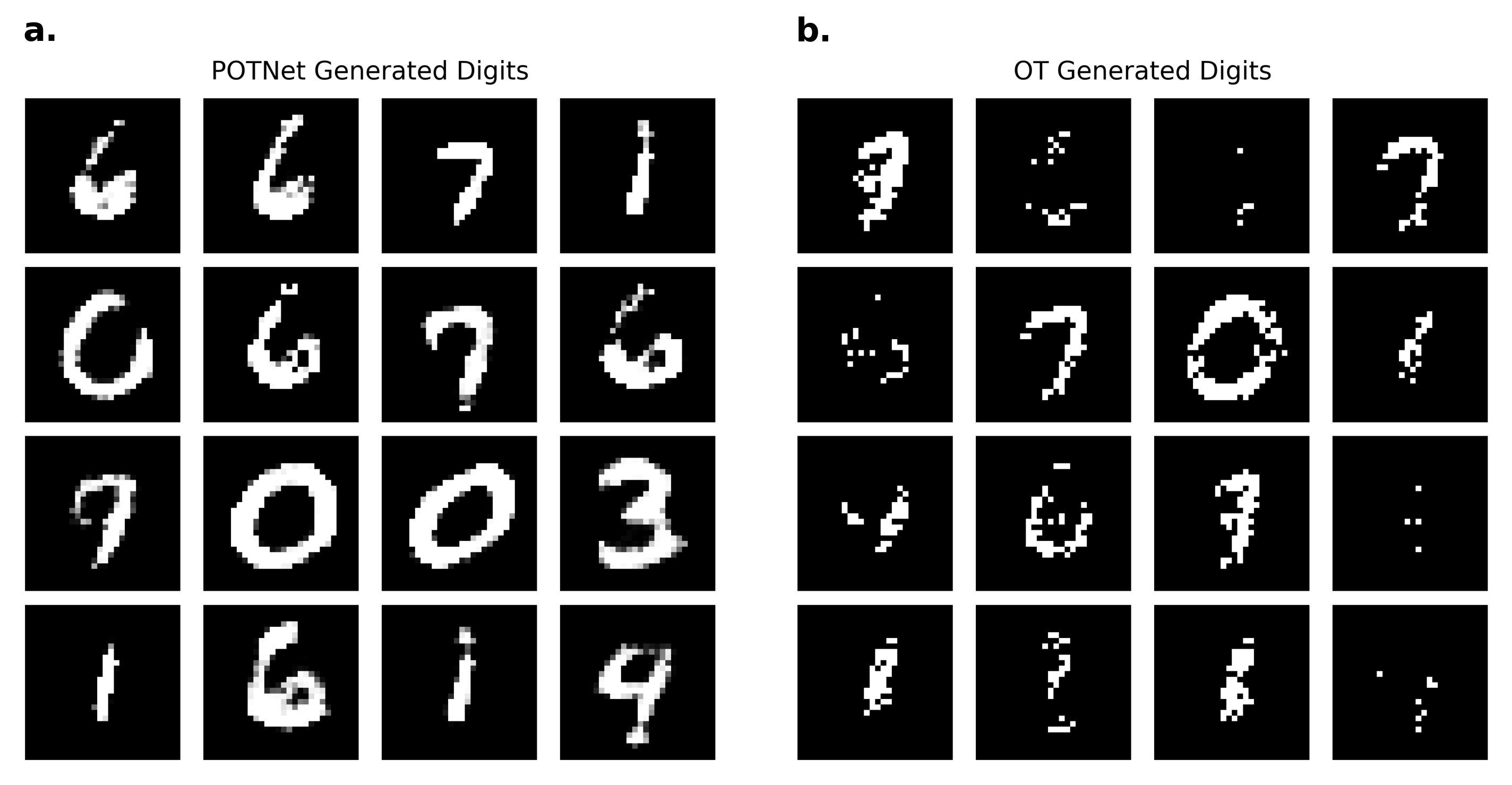}
\caption{Comparison of generated MNIST handwritten digits using a feedforward network architecture. Each model is trained for 300 epochs. (a). \potnet; (b). OT.}
\label{fig:mnist-ffnn}
\end{center}
\end{figure}

\section{Concluding Remarks}\label{Sec:disc}
To tackle the challenges of generating high-quality synthetic data, we introduce the novel Marginally-Penalized Wasserstein (\mpw) distance and propose the Penalized Optimal Transport Network (\potnet), a versatile generative model based on the \mpw distance.
The \mpw distance explicitly penalizes marginal mismatch along each dimension as a form of regularization to ensure proper joint alignment. It also naturally accommodates mixed data types, offers intuitive interpretations, and effectively captures multimodal distributions.
By leveraging low-dimensional marginal information, \potnet circumvents the curse of dimensionality typically experienced by the Wasserstein distance in high-dimensional settings.
Unlike WGANs which use duality-based distance approximation, \potnet employs the primal formulation of the Wasserstein distance, thereby eliminating the need for a critic network and extensive parameter tuning.
These modifications enable \potnet to significantly mitigate both mode dropping (type I mode collapse) and tail shrinkage (type II mode collapse). Our extensive simulation studies and real-world applications demonstrate \potnet's enhanced robustness and computational efficiency in generating high-quality synthetic samples. Additionally, we derive non-asymptotic bound on the generalization error of the \mpw loss and establish convergence rates of the generative distribution learned by \potnet. Our theoretical analysis also sheds light on the underpinnings of \potnet's efficacy in addressing mode collapse and producing desirable samples. These theoretical results substantiate the reliability and usefulness of the proposed model.

Several interesting avenues for future research emerge from this work. One natural direction is the application of the \mpw loss to other generative models, including flow-based models. This extension could potentially enhance the performance and stability of these models, especially in high-dimensional spaces where traditional approaches often struggle. Another compelling direction stems from our demonstration of the effectiveness and utility of low-dimensional marginal distributions. This insight could be leveraged to address a wide range of high-dimensional problems beyond generative modeling. For instance, it might prove valuable in dimensionality reduction techniques, feature selection methods, or in developing more robust clustering algorithms for high-dimensional data.

\section*{Acknowledgements}
The authors would like to thank Aldo Carranza, X.Y. Han, Chen Liu, and Julie Zhang for insightful discussions and valuable comments. 
During this work, W.S.L. is partially supported by the Stanford Data Science Graduate Fellowship and the Two Sigma Graduate Fellowship Fund. 
W.H.W.'s research is partially funded by the NSF grant 2310788.

\bibliographystyle{plainnat}
\bibliography{refs}

\newpage
\begin{appendices}

{\LARGE\bfseries\section*{Appendix}}
\noindent The Appendix section contains detailed proofs of theoretical results, supplementary empirical findings and plots, as well as the most general formulation of the model framework.

\section{Generalization of the \mpw Distance to Multi-coordinate Marginals}\label{appendix:k-coord-mpw}

\begin{defn}[$k$-Coordinate Marginal Distribution]
Consider any $k\in[d]$. For any $S=\{i_1,\cdots,i_k\}\subseteq [d]$ (where $i_1<\cdots<i_k$), we define the projection map $p_S:\mathbb{R}^d\rightarrow\mathbb{R}^{k}$ by setting $p_S(\pmb{x}):=(x_{i_1},\cdots,x_{i_k})$ for any $\pmb{x}=(x_1,x_2,\cdots,x_d)\in\mathbb{R}^d$. For any probability distribution $\mu$ on $\mathbb{R}^d$, we denote by $(p_S)_{*}\mu$ the pushforward of $\mu$ by $p_S$, i.e., the probability distribution on $\mathbb{R}^k$ such that for any Borel set $A\subseteq \mathbb{R}^k$, $(p_S)_{*}\mu(A)=\mu(p_S^{-1}(A))$. 
\end{defn}

\begin{defn}[Generalized \mpw Distance]
For any two probability distributions $\mu,\nu\in\mathcal{P}_1(\mathcal{X})$ and any collection $\mathcal{A}$ of subsets of $[d]$, we define the \emph{Generalized \mpw distance} between $\mu$ and $\nu$ as 
\begin{equation*}
    \mathcal{D}_{\mathcal{A}}(\mu,\nu):=W_1(\mu,\nu)+\sum_{S\in\mathcal{A}}\lambda_S W_1((p_S)_{*}\mu,(p_S)_{*}\nu),
\end{equation*}
where $\{\lambda_S\}_{S\in\mathcal{A}}$ are hyperparameters. 
\end{defn}

\section{Properties of the \mpw Distance}

In this section, we present additional properties of the \mpw Distance $\mathcal{D}(\cdot,\cdot)$.

\begin{defn}[Function Class for Dual Formulation]\label{defn_functions}
Let $\pmb{\lambda}=(\lambda_1,\cdots,\lambda_d)$ be given as in Definition \textcolor{red}{3.2} of the paper. We define the function class
\begin{align*}\label{def_fam}
    \mathcal{F}_{\pmb{\lambda}}:=\bigg\{\phi:\mathbb{R}^d\rightarrow\mathbb{R}:\, &\phi(\mathbf{x})=f(\mathbf{x})+\sum_{j=1}^d\lambda_j f_j(x_j)\text{ for any }\mathbf{x}=(x_1,x_2,\cdots,x_d)\in\mathbb{R}^d,\nonumber\\
    & \text{ where }f\in\mathrm{Lip}_{1,d}\text{ and }f_j\in\mathrm{Lip}_{1,1}\text{ for each }j\in [d]\bigg\}. 
\end{align*}
\end{defn}

\begin{thm}\label{dual_represent}
The \mpw distance $\mathcal{D}(\cdot,\cdot)$ is a metric on $\mathcal{P}_1(\mathcal{X})$. Moreover, it has the following dual representation:
for any $\mu,\mu'\in\mathcal{P}_1(\mathcal{X})$,
\begin{equation*}
    \mathcal{D}(\mu,\mu')=\sup\limits_{\phi\in\mathcal{F}_{\pmb{\lambda}}}\Big\{\int_{\mathbb{R}^d}\phi d\mu-\int_{\mathbb{R}^d}\phi d\mu'\Big\}.
\end{equation*}
\end{thm}
\begin{proof}[Proof of Theorem \ref{dual_represent}]

For any $\mu\in\mathcal{P}_1(\mathcal{X})$, 
\begin{equation*}
    \mathcal{D}(\mu,\mu)=W_1(\mu,\mu)+\sum_{j=1}^d \lambda_j W_1((p_j)_{*}\mu,(p_j)_{*}\mu)=0.
\end{equation*}
For any $\mu,\mu'\in\mathcal{P}_1(\mathcal{X})$,
\begin{align*}
    \mathcal{D}(\mu,\mu')&=W_1(\mu,\mu')+\sum_{j=1}^d \lambda_j W_1((p_j)_{*}\mu,(p_j)_{*}\mu')\\
    &=W_1(\mu',\mu)+\sum_{j=1}^d \lambda_j W_1((p_j)_{*}\mu',(p_j)_{*}\mu)=\mathcal{D}(\mu',\mu).
\end{align*}
If $\mu\neq \mu'$, we have $\mathcal{D}(\mu,\mu')\geq W_1(\mu,\mu')>0$. For any $\mu,\mu',\mu''\in\mathcal{P}_1(\mathcal{X})$, 
\begin{align*}
    \mathcal{D}(\mu,\mu'')&=W_1(\mu,\mu'')+\sum_{j=1}^d \lambda_j W_1((p_j)_{*}\mu,(p_j)_{*}\mu'')\nonumber\\
    &\leq (W_1(\mu,\mu')+W_1(\mu',\mu''))+\sum_{j=1}^d \lambda_j (W_1((p_j)_{*}\mu,(p_j)_{*}\mu')+W_1((p_j)_{*}\mu',(p_j)_{*}\mu''))\\
    &=\mathcal{D}(\mu,\mu')+\mathcal{D}(\mu',\mu'').
\end{align*}
Therefore, $\mathcal{D}(\cdot,\cdot)$ is a metric on $\mathcal{P}_1(\mathcal{X})$.

By the dual formulation of the $1$-Wasserstein distance \cite[Remark 6.5]{villani2009optimal}, for any two probability distributions $\mu,\mu'$ on $\mathbb{R}^d$, we have   
\begin{equation*}
    W_1(\mu,\mu')=\sup_{f\in \mathrm{Lip}_{1,d}}\Big\{\int_{\mathbb{R}^d} fd\mu-\int_{\mathbb{R}^d} fd\mu'\Big\},
\end{equation*}
and for any $j\in [d]$,
\begin{align*}
    W_1((p_j)_{*}\mu,(p_j)_{*}\mu')&=\sup_{f_j\in \mathrm{Lip}_{1,1}}\Big\{\int_{\mathbb{R}}f_jd((p_j)_{*}\mu)-\int_{\mathbb{R}}f_jd((p_j)_{*}\mu')\Big\}\nonumber\\
    &=\sup_{f_j\in \mathrm{Lip}_{1,1}}\Big\{\int_{\mathbb{R}^d}f_j\circ p_jd\mu-\int_{\mathbb{R}^d}f_j\circ p_jd\mu'\Big\}.
\end{align*}
Hence
\begin{align*}
    \mathcal{D}(\mu,\mu')&=W_1(\mu,\mu')+\sum_{j=1}^d\lambda_jW_1((p_j)_{*}\mu,(p_j)_{*}\mu')\nonumber\\
    &=\sup_{\substack{f\in\mathrm{Lip}_{1,d},\\f_j\in\mathrm{Lip}_{1,1}\text{ for every }j\in [d]}}\Big\{\int_{\mathbb{R}^d}\big(f+\sum_{j=1}^d\lambda_j f_j\circ p_j\big)d\mu-\int_{\mathbb{R}^d}\big(f+\sum_{j=1}^d\lambda_j f_j\circ p_j\big)d\mu'\Big\}\nonumber\\
    &=\sup_{\phi\in\mathcal{F}_{\pmb{\lambda}}} \Big\{\int_{\mathbb{R}^d}\phi d\mu-\int_{\mathbb{R}^d}\phi d\mu'\Big\}. 
 \end{align*}

\end{proof}

\section{Conditional \potnet}

The theoretical basis for conditional generative modeling is the noise-outsourcing lemma \cite{zhou2023deep}. The pseudo-code for conditional \potnet is provided in Algorithm \ref{alg:ConditionPOTNet}, where $\widehat{P}_X^{(t),b}$ is the empirical distribution formed by $\mathcal{B}_X^{(t),b}$ and $\widehat{P}_{G_{\theta}(\tilde{Z})}^{(t),b}$ is the empirical distribution formed by $\{G_{\theta}(\tilde{Z}^{(t),i}):i\in\mathcal{I}_b\}$.

\begin{algorithm}[h!]
 \SetAlgoLined
\caption{Training Procedure for Conditional \potnet}
\label{alg:ConditionPOTNet}
\KwIn{Data $(X^{(1)},Y^{(1)}),\cdots,(X^{(n)},Y^{(n)})$ ($Y^{(i)}$ are conditional features),\\ latent dimension $d$, dimension of conditional features $d_c$, batch size $m$,\\ regularization parameter $\pmb{\lambda} = (\lambda_1, \dots, \lambda_d)$.}
\KwOut{Optimized parameters $\theta$.}

\For{ iteration $t = 1, \dots, T$ }{
     \tcp*[h]{Permute indices}\; 
     $\pi_t: [n]\to[n] \in \mathbb{S}_n$\;
     \vspace{0.1cm}
    \For{$b = 1, \dots, \ceil{n/m}$}{
        Set $\mathcal{I}_b\gets\{i\in [n]:(b-1)m+1\leq i\leq bm\}$\; 
        Form minibatch of real observations $\mathcal{B}_X^{(t),b}\gets\{X^{(\pi_t(i))}:i\in\mathcal{I}_b\}$\; 
        Obtain the corresponding conditional features $\{Y^{(\pi_t(i))}\}_{i\in\mathcal{I}_b}$\;
        Sample latent variables $\{Z^{(t),i}\}_{i\in \mathcal{I}_b}\sim P_Z(\cdot)$\;
        $\{\tilde{Z}^{(t),i}\}_{i\in \mathcal{I}_b}\gets$ concatenate conditional features $\{Y^{(\pi_t(i))}\}_{i\in\mathcal{I}_b}$ with source of noise $\{Z^{(t),i}\}_{i\in \mathcal{I}_b}$\;
        Form minibatch $\mathcal{B}_{\tilde{Z}}^{(t),b}\gets\{\tilde{Z}^{(t),i}:i\in\mathcal{I}_b\}$\; 
        \vspace{0.1cm}
        \tcp*[h]{Compute gradient}\;
        $G_{\theta} \gets \nabla_\theta \left[ W_1\left(\widehat{P}_X^{(t),b},\widehat{P}_{G_{\theta}(\tilde{Z})}^{(t),b} \right)+\sum_{j=1}^d\lambda_jW_1\left((p_j)_{*}\widehat{P}_X^{(t),b}, (p_j)_{*}\widehat{P}_{G_{\theta}(\tilde{Z})}^{(t),b}\right)\right]$\;
        \vspace{0.1cm}
        \tcp*[h]{Update parameters}\;
         $\theta \gets \theta - \text{AdamW}(\theta, G_\theta)$\; 
    }
}
\end{algorithm}

\section{Proofs of Theoretical Results from Section 4}\label{Sec: proofs}

In this section, we present the proofs of Theorems 4.1-4.4 and Corollary 4.3.1 from Section 4. 

\subsection{Proof of Theorem 4.1}

In this subsection, we present the proof of Theorem 4.1. 

\begin{proof}[Proof of Theorem 4.1]

By Theorem \ref{dual_represent}, we have
\begin{align}\label{E2}
   &\mathcal{D}(P_X,P_{\theta})-\mathcal{D}(\PX,P_{\theta})\nonumber\\
   =&\sup_{\phi\in\mathcal{F}_{\pmb{\lambda}}} \Big\{\int_{\mathbb{R}^d}\phi dP_X-\int_{\mathbb{R}^d}\phi dP_{\theta}\Big\}-\sup_{\phi\in\mathcal{F}_{\pmb{\lambda}}} \Big\{\int_{\mathbb{R}^d}\phi d\PX-\int_{\mathbb{R}^d}\phi dP_{\theta}\Big\}\nonumber\\
   \leq& \sup_{\phi\in\mathcal{F}_{\pmb{\lambda}}} \Big\{\Big(\int_{\mathbb{R}^d}\phi dP_X-\int_{\mathbb{R}^d}\phi dP_{\theta}\Big)-\Big(\int_{\mathbb{R}^d}\phi d\PX-\int_{\mathbb{R}^d}\phi dP_{\theta}\Big)\Big\}\nonumber\\
   =& \sup_{\phi\in\mathcal{F}_{\pmb{\lambda}}} \Big\{\int_{\mathbb{R}^d}\phi dP_X-\int_{\mathbb{R}^d}\phi d\PX\Big\}. 
\end{align}

Recall that $B_M=\{\mathbf{x}\in\mathbb{R}^d:\|\mathbf{x}\|_2\leq M\}$. For any $x^{(1)},\cdots,x^{(n)}\in B_M$, we define
\begin{equation*}
    H(x^{(1)},\cdots,x^{(n)}):=\sup_{\phi\in\mathcal{F}_{\pmb{\lambda}}} \Big\{\int_{\mathbb{R}^d}\phi dP_X-\frac{1}{n}\sum_{i=1}^n\phi(x^{(i)}) \Big\}.
\end{equation*}
Note that 
\begin{equation}\label{E2.1}
    \sup_{\phi\in\mathcal{F}_{\pmb{\lambda}}} \Big\{\int_{\mathbb{R}^d}\phi dP_X-\int_{\mathbb{R}^d}\phi d\PX\Big\}=H(X^{(1)},\cdots,X^{(n)}).
\end{equation}
For any $i\in [n]$ and $y^{(i)}\in B_M$, by the Cauchy-Schwarz inequality,
\begin{align}\label{E2.2}
    & \big|H(x^{(1)},\cdots,x^{(i)},\cdots,x^{(n)})-H(x^{(1)},\cdots,y^{(i)},\cdots,x^{(n)})\big|\leq \frac{1}{n}\sup_{\phi\in\mathcal{F}_{\pmb{\lambda}}}|\phi(x^{(i)})-\phi(y^{(i)})| \nonumber\\
    \leq& \frac{1}{n}\big\|x^{(i)}-y^{(i)}\big\|_2+\frac{1}{n}\sum_{j=1}^d \lambda_j \big|x^{(i)}_j-y^{(i)}_j\big|\leq \frac{1+\|\pmb{\lambda}\|_2}{n}\big\|x^{(i)}-y^{(i)}\big\|_2\leq \frac{2(1+\|\pmb{\lambda}\|_2)M}{n}.
\end{align}
Hereafter, we denote by $\mathbb{P}_X$ and $\mathbb{E}_X$ the probability and expectation operators with respect to $X^{(1)},\cdots,X^{(n)}\sim P_X$. 
By \eqref{E2.1}, \eqref{E2.2}, and McDiarmid's inequality (see, e.g., \cite[Theorem 6]{BLM2013}), for any $t\geq 0$,
\begin{align*}
    & \mathbb{P}_X\bigg(\sup_{\phi\in\mathcal{F}_{\pmb{\lambda}}} \Big\{\int_{\mathbb{R}^d}\phi dP_X-\int_{\mathbb{R}^d}\phi d\PX\Big\}> \mathbb{E}_X\Big[\sup_{\phi\in\mathcal{F}_{\pmb{\lambda}}} \Big\{\int_{\mathbb{R}^d}\phi dP_X-\int_{\mathbb{R}^d}\phi d\PX\Big\}\Big]+t\bigg)\nonumber\\
    =& \mathbb{P}_X\big(H(X^{(1)},\cdots,X^{(n)})> \mathbb{E}_X\big[H(X^{(1)},\cdots,X^{(n)})\big]+t\big)\leq \exp\Big(-\frac{n t^2}{2(1+\|\pmb{\lambda}\|_2)^2M^2}\Big).
\end{align*}
Taking $t=M(1+\|\pmb{\lambda}\|_2)\sqrt{2\log(\delta^{-1})\slash n}$, we obtain that with probability at least $1-\delta$,
\begin{align}\label{E2.3}
    &\sup_{\phi\in\mathcal{F}_{\pmb{\lambda}}} \Big\{\int_{\mathbb{R}^d}\phi dP_X-\int_{\mathbb{R}^d}\phi d\PX\Big\}\nonumber\\
    \leq & \mathbb{E}_X\Big[\sup_{\phi\in\mathcal{F}_{\pmb{\lambda}}} \Big\{\int_{\mathbb{R}^d}\phi dP_X-\int_{\mathbb{R}^d}\phi d\PX\Big\}\Big]+M(1+\|\pmb{\lambda}\|_2)\sqrt{\frac{2\log(\delta^{-1})}{n}}.
\end{align}

By the symmetrization inequality (see, e.g., \cite[Lemma 2.3.1]{VaartWellner}), we have
\begin{align}\label{E2.5}
    & \mathbb{E}_X\bigg[\sup_{\phi\in\mathcal{F}_{\pmb{\lambda}}} \Big\{\int_{\mathbb{R}^d}\phi dP_X-\int_{\mathbb{R}^d}\phi d\PX\Big\}\bigg]\nonumber\\
    =&\mathbb{E}_X\bigg[\sup_{\phi\in\mathcal{F}_{\pmb{\lambda}}} \Big\{\int_{\mathbb{R}^d}\big(\phi(x)-\phi(0)\big) dP_X(x)-\int_{\mathbb{R}^d}\big(\phi(x)-\phi(0)\big) d\PX(x)\Big\}\bigg]\nonumber\\
    \leq& \frac{2}{n} \mathbb{E}_{X,\pmb{\epsilon}}\bigg[\sup_{\phi\in\mathcal{F}_{\pmb{\lambda}}}\Big\{\sum_{i=1}^n \epsilon_i\big(\phi(X^{(i)})-\phi(0)\big)\Big\}\bigg]=2\mathbb{E}_X[\hat{R}_n(\mathcal{F}_{\pmb{\lambda}})],
\end{align}
in which $\pmb{\epsilon}:=(\epsilon_i)_{i=1}^n$ and $\{\epsilon_i\}_{i=1}^n$ are i.i.d. Rademacher random variables.  

Now for any $x^{(1)},\cdots,x^{(n)}\in B_M$, we let 
\begin{equation*}
   \tilde{H}(x^{(1)},\cdots,x^{(n)}):=n^{-1}\mathbb{E}_{\pmb{\epsilon}}\Big[\sup_{\phi\in\mathcal{F}_{\pmb{\lambda}}}\Big\{\sum_{i=1}^n\epsilon_i \big(\phi(x^{(i)})-\phi(0)\big)\Big\}\Big].
\end{equation*}
Note that $\hat{R}_n(\mathcal{F}_{\pmb{\lambda}})=\tilde{H}(X^{(1)},\cdots,X^{(n)})$. For any $i\in [n]$ and $y^{(i)}\in B_M$, arguing as in \eqref{E2.2}, we obtain that 
\begin{align*}
   &  \big|\tilde{H}(x^{(1)},\cdots,x^{(i)},\cdots,x^{(n)})-\tilde{H}(x^{(1)},\cdots,y^{(i)},\cdots,x^{(n)})\big|\nonumber\\
   \leq &\frac{1}{n}\sup_{\phi\in\mathcal{F}_{\pmb{\lambda}}}|\phi(x^{(i)})-\phi(y^{(i)})| \leq \frac{2(1+\|\pmb{\lambda}\|_2)M}{n}.
\end{align*}
Hence by a similar argument that leads to \eqref{E2.3}, we obtain that with probability at least $1-\delta$, 
\begin{equation}\label{E2.4}
    \hat{R}_n(\mathcal{F}_{\pmb{\lambda}})\geq \mathbb{E}_X[\hat{R}_n(\mathcal{F}_{\pmb{\lambda}})]-M(1+\|\pmb{\lambda}\|_2)\sqrt{\frac{2\log(\delta^{-1})}{n}}.
\end{equation}

Combining \eqref{E2} and \eqref{E2.3}-\eqref{E2.4}, we conclude that with probability at least $1-2\delta$, 
\begin{equation*}
    \mathcal{D}(P_X,P_{\theta})-\mathcal{D}(\PX,P_{\theta})\leq 2\hat{R}_n(\mathcal{F}_{\pmb{\lambda}})+3M(1+\|\pmb{\lambda}\|_2)\sqrt{\frac{2\log(\delta^{-1})}{n}}. 
\end{equation*}

\end{proof}

\subsection{Proof of Theorem 4.2}

The proof of Theorem 4.2 relies on the following result on the convergence rate of the empirical measure $\PX$ to the true data generating distribution $P_X$. We refer the reader to \cite[Chapter 2]{chewi2024statistical} for a proof of this result and a review of related literature; see also \cite{fournier2015rate, weed2019sharp}.   

\begin{prop}\label{Pro2.1}
Assume that the support of $P_X$ is contained in $B_M=\{\mathbf{x}\in\mathbb{R}^d:\|\mathbf{x}\|_2\leq M\}$. Then there exists a deterministic constant $C$ that only depends on $M$, such that 
\begin{equation*}
    \mathbb{E}[W_1(\PX,P_X)]\leq C\sqrt{d}\cdot 
    \begin{cases}
        n^{-1\slash 2} & \text{ if } d=1\\
        (\log{n}\slash n)^{1\slash 2} & \text{ if }d=2\\
        n^{-1\slash d} & \text{ if }d\geq 3
    \end{cases}, 
\end{equation*}
and for any $j\in [d]$,
\begin{equation*}
    \mathbb{E}[W_1(\PXj,P_{X,j})]\leq C n^{-1\slash 2}.
\end{equation*}
\end{prop}

Based on Proposition \ref{Pro2.1}, we complete the proof of Theorem 4.2 as follows. 

\begin{proof}[Proof of Theorem 4.2]

By the definition of $\hat{\theta}$, we have
\begin{equation*}
    W_1(\PX,P_{\hat{\theta}})+\sum_{j=1}^d \lambda_j W_1(\PXj,P_{\hat{\theta},j}) =\mathcal{D}(\PX,P_{\hat{\theta}})\leq \inf_{\theta\in\Theta}\{\mathcal{D}(\PX,P_{\theta})\}. 
\end{equation*}
Hence by Proposition \ref{Pro2.1},
\begin{align*}
     \mathbb{E}[W_1(P_X,P_{\hat{\theta}})]\leq& \mathbb{E}[W_1(\PX,P_{\hat{\theta}})]+\mathbb{E}[W_1(\PX,P_X)]\nonumber\\
    \leq & \mathbb{E}\big[\inf_{\theta\in\Theta}\{\mathcal{D}(\PX,P_{\theta})\}\big]+C\sqrt{d}\cdot 
    \begin{cases}
        n^{-1\slash 2} & \text{ if } d=1\\
        (\log{n}\slash n)^{1\slash 2} & \text{ if }d=2\\
        n^{-1\slash d} & \text{ if }d\geq 3
    \end{cases},
\end{align*}
and for any $j\in [d]$,
\begin{align*}
    \mathbb{E}[W_1(P_{X,j},P_{\hat{\theta},j})]&\leq \mathbb{E}[W_1(\PXj,P_{\hat{\theta},j})]+\mathbb{E}[W_1(\PXj,P_{X,j})]\nonumber\\
    &\leq \lambda_j^{-1}\mathbb{E}\big[\inf_{\theta\in\Theta}\{\mathcal{D}(\PX,P_{\theta})\}\big]+C n^{-1\slash 2},
\end{align*}
where $C$ is a deterministic constant that only depends on $M$. 
    
\end{proof}

\subsection{Proofs of Theorems 4.3-4.4 and Corollary 4.3.1}\label{Sect.4.3}

In this subsection, we present the proofs of Theorems 4.3-4.4 and Corollary 4.3.1. 

We start with the following two lemmas. For any function $f:\mathbb{R}\rightarrow\mathbb{R}$, we denote by $\|f\|_{\mathrm{Lip}}$ the Lipschitz norm of $f$.

\begin{lemma}\label{Lem:KS}
For any two probability distributions $\mu$ and $\mu'$ on $\mathbb{R}$ and any $t\in\mathbb{R},\gamma>0$, we have
\begin{equation}\label{KSE1}
  \mu((-\infty,t-\gamma])-\gamma^{-1} W_1(\mu,\mu')\leq  \mu'((-\infty,t])\leq \mu((-\infty,t+\gamma])+ \gamma^{-1}W_1(\mu,\mu').
\end{equation}
\end{lemma}
\begin{proof}
By the dual formulation of the $1$-Wasserstein distance \cite[Remark 6.5]{villani2009optimal}, we have  
\begin{equation}\label{dual.W1}
    W_1(\mu,\mu')=\sup_{f\in\mathrm{Lip}_{1,1}}\Big\{\int_{\mathbb{R}} fd\mu-\int_{\mathbb{R}} fd\mu'\Big\}, 
\end{equation}
where $\mathrm{Lip}_{1,1}$ is the set of $1$-Lipschitz functions on $\mathbb{R}$. For any $t\in\mathbb{R}$, $\gamma>0$, and $x\in\mathbb{R}$, we define
\begin{equation*}
    \psi_{t,\gamma}(x) := \begin{cases}
        1 & x\leq t\\
         1-\frac{x-t}{\gamma} & t<x<t+\gamma\\
        0 & x\geq t+\gamma
    \end{cases}. 
\end{equation*}
Note that  $\|\psi_{t,\gamma}\|_{\mathrm{Lip}}\leq \gamma^{-1}$. Hence by \eqref{dual.W1},
\begin{equation}\label{E4.1}
    \Big|\int_{\mathbb{R}}\psi_{t,\gamma}d\mu-\int_{\mathbb{R}}\psi_{t,\gamma}d\mu'\Big|\leq \gamma^{-1} W_1(\mu,\mu'). 
\end{equation}
Now note that
\begin{equation}\label{E4.2}
  \mu((-\infty,t])-\mu'((-\infty,t+\gamma]) \leq \int_{\mathbb{R}}\psi_{t,\gamma}d\mu-\int_{\mathbb{R}}\psi_{t,\gamma}d\mu'\leq \mu((-\infty,t+\gamma])-\mu'((-\infty,t]).
\end{equation}
Combining \eqref{E4.1} and \eqref{E4.2}, we have
\begin{equation}\label{E4.5}
    \mu((-\infty,t+\gamma])-\mu'((-\infty,t])\geq -\gamma^{-1} W_1(\mu,\mu'),
\end{equation}
\begin{equation}\label{E4.3}
  \mu((-\infty,t])-\mu'((-\infty,t+\gamma])\leq \gamma^{-1} W_1(\mu,\mu'). 
\end{equation}
Replacing $t$ by $t-\gamma$ in \eqref{E4.3}, we obtain that
\begin{equation}\label{E4.4}
    \mu((-\infty,t-\gamma])-\mu'((-\infty,t])\leq  \gamma^{-1} W_1(\mu,\mu'). 
\end{equation}
\eqref{KSE1} follows from \eqref{E4.5} and \eqref{E4.4}. 

\end{proof}

\begin{lemma}\label{Lemma2.2}
For any $\delta\in (0,1)$, let $\mathcal{E}_{\delta}$ be the event that for all $j\in [d]$,
\begin{align}\label{E4.6}
   & \sup_{t\in\mathbb{R}}\big|\PXj((-\infty,t])-P_{X,j}((-\infty,t])\big|\leq \sqrt{\frac{\log(4d\slash \delta)}{2n}},\nonumber\\
   & \sup_{t\in\mathbb{R}}\big|\PXj([t,\infty))-P_{X,j}([t,\infty))\big|\leq \sqrt{\frac{\log(4d\slash \delta)}{2n}}.
\end{align}
Then we have $\mathbb{P}(\mathcal{E}_{\delta}^c)\leq\delta$. Moreover, when the event $\mathcal{E}_{\delta}$ holds, for any $j\in [d]$ and $t>0$, 
\begin{align}
    \big|\PXj((-t,t)^c)-P_{X,j}((-t,t)^c)\big|\leq \sqrt{\frac{2\log(4d\slash \delta)}{n}};
\end{align}
for any $j\in [d]$ and $t_1<t_2$, 
\begin{align}
    \big|\PXj([t_1,t_2])-P_{X,j}([t_1,t_2])\big|\leq \sqrt{\frac{2\log(4d\slash \delta)}{n}}. 
\end{align}
\end{lemma}

\begin{proof}

By the Dvoretzky–Kiefer–Wolfowitz inequality \citep{dvoretzky1956asymptotic, massart1990tight}, for any $j\in [d]$ and $s\geq 0$,
\begin{equation}\label{E4.10}
    \mathbb{P}\big(\sup_{t\in\mathbb{R}}\big|\PXj((-\infty,t])-P_{X,j}((-\infty,t])\big|> s\big)\leq 2\exp(-2n s^2). 
\end{equation}
Similarly, for any $j\in [d]$ and $s\geq 0$,
\begin{equation}\label{E4.10n}
    \mathbb{P}\big(\sup_{t\in\mathbb{R}}\big|\PXj([t,\infty))-P_{X,j}([t,\infty))\big|> s\big)\leq 2\exp(-2n s^2). 
\end{equation}
Taking $s=\sqrt{\log(4d\slash \delta)\slash 2n}$ in \eqref{E4.10}-\eqref{E4.10n} and applying the union bound, we obtain that $\mathbb{P}(\mathcal{E}_{\delta}^c)\leq\delta$. 

Now assume that the event $\mathcal{E}_{\delta}$ holds. For any $j\in [d]$ and $t>0$, 
\begin{align*}
   & \big|\PXj((-t,t)^c)-P_{X,j}((-t,t)^c)\big|\nonumber\\
   =& \big|\PXj((-\infty,-t])+\PXj([t,\infty))-P_{X,j}((-\infty,-t])-P_{X,j}([t,\infty))\big| \nonumber\\
   \leq& \big|\PXj((-\infty,-t])-P_{X,j}((-\infty,-t])\big|+\big|\PXj([t,\infty))-P_{X,j}([t,\infty))\big|\nonumber\\
   \leq& 2\sqrt{\frac{\log(4d\slash \delta)}{2n}}=\sqrt{\frac{2\log(4d\slash \delta)}{n}}.
\end{align*}
For any $j\in [d]$ and $t_1<t_2$, 
\begin{align*}
    & \big|\PXj([t_1,t_2])-P_{X,j}([t_1,t_2])\big|=\big|\PXj([t_1,t_2]^c)-P_{X,j}([t_1,t_2]^c)\big|\nonumber\\
   =& \big|\PXj((-\infty,t_1))+\PXj((t_2,\infty))-P_{X,j}((-\infty,t_1))-P_{X,j}((t_2,\infty))\big| \nonumber\\
   \leq& \big|\PXj((-\infty,t_1))-P_{X,j}((-\infty,t_1))\big|+\big|\PXj((t_2,\infty))-P_{X,j}((t_2,\infty))\big|\nonumber\\
   =& \big|\PXj([t_1,\infty))-P_{X,j}([t_1,\infty))\big|+\big|\PXj((-\infty,t_2])-P_{X,j}((-\infty,t_2])\big|\nonumber\\
   \leq& 2\sqrt{\frac{\log(4d\slash \delta)}{2n}}=\sqrt{\frac{2\log(4d\slash \delta)}{n}}.
\end{align*}

\end{proof}

\begin{proof}[Proof of Theorem 4.3]

We fix any $\delta\in (0,1)$, and take $\mathcal{E}_{\delta,1}$ to be the event that for all $j\in [d]$ and $t_1<t_2$, 
\begin{align}\label{E4.13}
   \big|\PXj([t_1,t_2])-P_{X,j}([t_1,t_2])\big|\leq \sqrt{\frac{2\log(8d\slash \delta)}{n}}.
\end{align}
By Lemma \ref{Lemma2.2}, we have $\mathbb{P}(\mathcal{E}_{\delta,1}^c)\leq \delta\slash 2$.

We define
\begin{align*}
    \mathcal{K}:=([-\Delta-r,-\Delta+r]\cup [\Delta-r,\Delta+r])^d,\quad \mathcal{K}_0:=[-\Delta-r,-\Delta+r]^d\cup [\Delta-r,\Delta+r]^d,
\end{align*}
and let 
$\mathcal{E}_{\delta,2}$ be the event that 
\begin{equation}\label{E4.20}
    \PX(\mathbb{R}^d\backslash\mathcal{K}_0)\leq d\psi(r)+\sqrt{\frac{\log(8\slash \delta)}{2n}}.
\end{equation}
Define $W:=\PX(\mathbb{R}^d\backslash\mathcal{K}_0)=n^{-1}\sum_{i=1}^n\mathbbm{1}_{X^{(i)}\in\mathbb{R}^d\backslash\mathcal{K}_0}$. By Hoeffding's inequality (see, e.g., \cite[Theorem 2.8]{BLM2013}), for any $t\geq 0$,
\begin{equation*}
    \mathbb{P}(W-\mathbb{E}[W]\geq t)\leq \exp(-2n t^2),
\end{equation*}
where $\mathbb{E}[W]= P_X(\mathbb{R}^d\backslash \mathcal{K}_0)\leq \sum_{j=1}^d\rho_j([-r,r]^c)\leq d\psi(r)$. Taking $t=\sqrt{\log(8\slash \delta)\slash 2n}$, we obtain that $\mathbb{P}(\mathcal{E}_{\delta,2}^c)\leq\delta\slash 8$. 

Similarly, letting $\mathcal{E}_{\delta,3}$ be the event that
\begin{align}\label{E4.23}
    &|\PX([-\Delta-r,-\Delta+r]^d)-P_X([-\Delta-r,-\Delta+r]^d)|\leq \sqrt{\frac{\log(16\slash \delta)}{2n}},\nonumber\\
    &|\PX([\Delta-r,\Delta+r]^d)-P_X([\Delta-r,\Delta+r]^d)|\leq \sqrt{\frac{\log(16\slash \delta)}{2n}},
\end{align}
we can deduce using Hoeffding's inequality that $\mathbb{P}(\mathcal{E}_{\delta,3}^c)\leq\delta\slash 4$.

Throughout the rest of the proof, we assume that the event $\mathcal{E}_{\delta,1}\cap \mathcal{E}_{\delta,2} \cap \mathcal{E}_{\delta,3}$ holds.

Consider any $k\in [2]$ and $j\in [d]$. Recall that $\mu_{k,j}$ is the $j$th coordinate of $\pmb{\mu}_k$. By the definition of $\hat{\theta}$ (the solution to the constrained minimization problem in equation (\textcolor{red}{4.3}) of the paper), we have $W_1(\PXj,P_{\hat{\theta},j})\leq \epsilon_j$. From Lemma \ref{Lem:KS}, we can deduce that for any $\gamma>0$, 
\begin{align}
 & \PXj([\mu_{k,j}-r+\gamma,\mu_{k,j}+r-\gamma])-2\gamma^{-1}\epsilon_j\nonumber\\
  \leq &  P_{\hat{\theta},j}([\mu_{k,j}-r,\mu_{k,j}+r])\leq \PXj([\mu_{k,j}-r-\gamma,\mu_{k,j}+r+\gamma])+2\gamma^{-1}\epsilon_j.
\end{align}
Hence
\begin{align*}
    \big|P_{\hat{\theta},j}([\mu_{k,j}-r,\mu_{k,j}+r])-P_{X,j}([\mu_{k,j}-r,\mu_{k,j}+r])\big| \leq 2\gamma^{-1}\epsilon_j+2M\gamma+\sqrt{\frac{2\log(8d\slash \delta)}{n}}. 
\end{align*}
Taking $\gamma=\sqrt{\epsilon_j\slash M}$, we obtain that
\begin{align}\label{E4.14}
    \big|P_{\hat{\theta},j}([\mu_{k,j}-r,\mu_{k,j}+r])-P_{X,j}([\mu_{k,j}-r,\mu_{k,j}+r])\big| \leq 4\sqrt{M\epsilon_j}+\sqrt{\frac{2\log(8d\slash \delta)}{n}}.
\end{align}
Therefore, we have
\begin{align}\label{E4.15}
   & \big|P_{\hat{\theta},j}(([-\Delta-r,-\Delta+r]\cup [\Delta-r,\Delta+r])^c)\nonumber\\
   & -P_{X,j}(([-\Delta-r,-\Delta+r]\cup [\Delta-r,\Delta+r])^c)\big|\nonumber\\
   =& \big|P_{\hat{\theta},j}([-\Delta-r,-\Delta+r]\cup [\Delta-r,\Delta+r])\nonumber\\
   &-P_{X,j}([-\Delta-r,-\Delta+r]\cup [\Delta-r,\Delta+r])\big|\nonumber\\
   =& \big|P_{\hat{\theta},j}([-\Delta-r,-\Delta+r])+P_{\hat{\theta},j}([\Delta-r,\Delta+r])\nonumber\\
   &-P_{X,j}([-\Delta-r,-\Delta+r])-P_{X,j}([\Delta-r,\Delta+r])\big|\nonumber\\
   \leq & \big|P_{\hat{\theta},j}(([-\Delta-r,-\Delta+r])-P_{X,j}(([-\Delta-r,-\Delta+r])\big|\nonumber\\
   & + \big|P_{\hat{\theta},j}(([\Delta-r,\Delta+r])-P_{X,j}([\Delta-r,\Delta+r])\big|\nonumber\\
   \leq & 8\sqrt{M\epsilon_j}+2\sqrt{\frac{2\log(8d\slash \delta)}{n}}.
\end{align}

Note that
\begin{equation*}
    \mathbb{R}^d\backslash\mathcal{K}=\bigcup_{j=1}^d \{\mathbf{x}=(x_1,\cdots,x_j)\in\mathbb{R}^d:x_j\in([-\Delta-r,-\Delta+r]\cup [\Delta-r,\Delta+r])^c \}.
\end{equation*}
Hence by \eqref{E4.15} and the union bound,
\begin{align}\label{E4.22}
    P_{\hat{\theta}}(\mathbb{R}^d\backslash\mathcal{K})\leq&\sum_{j=1}^d P_{\hat{\theta}}(\{\mathbf{x}=(x_1,\cdots,x_j)\in\mathbb{R}^d:x_j\in([-\Delta-r,-\Delta+r]\cup [\Delta-r,\Delta+r])^c \})\nonumber\\
    =&\sum_{j=1}^d P_{\hat{\theta},j}(([-\Delta-r,-\Delta+r]\cup [\Delta-r,\Delta+r])^c)\nonumber\\
    \leq&\sum_{j=1}^d P_{X,j}(([-\Delta-r,-\Delta+r]\cup [\Delta-r,\Delta+r])^c)\nonumber\\
    &+8d\sqrt{M\max_{j\in[d]}\{\epsilon_j\}}+2d\sqrt{\frac{2\log(8d\slash \delta)}{n}}\nonumber\\
    \leq&\sum_{j=1}^d\rho_{j}([-r,r]^c)+8d\sqrt{M\max_{j\in[d]}\{\epsilon_j\}}+2d\sqrt{\frac{2\log(8d\slash \delta)}{n}}\nonumber\\
    \leq& d\psi(r)+8d\sqrt{M\max_{j\in[d]}\{\epsilon_j\}}+2d\sqrt{\frac{2\log(8d\slash \delta)}{n}}.
\end{align}

By the definition of $\hat{\mathcal{R}}$, 
\begin{align}\label{E4.16}
    W_1(\PX,P_{\hat{\theta}})\leq \hat{\mathcal{R}}. 
\end{align}
Let $\Upsilon$ be the optimal coupling of $P_{\hat{\theta}}$ and $\PX$ on $\mathbb{R}^d\times\mathbb{R}^d$, such that
\begin{equation}\label{E4.17}
    \int_{\mathbb{R}^d\times\mathbb{R}^d} \|\mathbf{x}-\mathbf{y}\|_2 d\Upsilon(\mathbf{x},\mathbf{y}) = W_1(\PX,P_{\hat{\theta}}). 
\end{equation}
For any $(\mathbf{x},\mathbf{y})\in [-\Delta-r,-\Delta+r]^d\times ((\mathcal{K}\backslash \mathcal{K}_0)\cup [\Delta-r,\Delta+r]^d)$, we have $\|\mathbf{x}-\mathbf{y}\|_2\geq 2(\Delta-r)$. Hence by \eqref{E4.16} and \eqref{E4.17},
\begin{equation}\label{E4.18}
    \Upsilon([-\Delta-r,-\Delta+r]^d\times ((\mathcal{K}\backslash \mathcal{K}_0)\cup [\Delta-r,\Delta+r]^d))\leq \frac{\hat{\mathcal{R}}}{ 2(\Delta-r)}. 
\end{equation}
Similarly,
\begin{equation}\label{E4.19}
    \Upsilon(
    ((\mathcal{K}\backslash \mathcal{K}_0)\cup [\Delta-r,\Delta+r]^d)
    \times [-\Delta-r,-\Delta+r]^d)\leq \frac{\hat{\mathcal{R}}}{ 2(\Delta-r)}.
\end{equation}

Now by \eqref{E4.20} and \eqref{E4.18},
\begin{align*}
   & P_{\hat{\theta}}([-\Delta-r,-\Delta+r]^d)-\PX([-\Delta-r,-\Delta+r]^d) \nonumber\\
   \leq& P_{\hat{\theta}}([-\Delta-r,-\Delta+r]^d)-\Upsilon([-\Delta-r,-\Delta+r]^d\times[-\Delta-r,-\Delta+r]^d)\nonumber\\
   =&\Upsilon([-\Delta-r,-\Delta+r]^d\times((\mathbb{R}^d\backslash\mathcal{K})\cup (\mathcal{K}\backslash\mathcal{K}_0)\cup[\Delta-r,\Delta+r]^d))\nonumber\\
   \leq& \Upsilon([-\Delta-r,-\Delta+r]^d\times(\mathbb{R}^d\backslash\mathcal{K}))+\Upsilon([-\Delta-r,-\Delta+r]^d\times( (\mathcal{K}\backslash\mathcal{K}_0)\cup[\Delta-r,\Delta+r]^d))\nonumber\\
 \leq& \PX(\mathbb{R}^d\backslash\mathcal{K})+ \Upsilon([-\Delta-r,-\Delta+r]^d\times( (\mathcal{K}\backslash\mathcal{K}_0)\cup[\Delta-r,\Delta+r]^d))\nonumber\\
 \leq& d\psi(r)+\sqrt{\frac{\log(8\slash \delta)}{2n}}+\frac{\hat{\mathcal{R}}}{ 2(\Delta-r)}.
\end{align*}
Similarly, by \eqref{E4.22} and \eqref{E4.19},
\begin{align*}
     & \PX([-\Delta-r,-\Delta+r]^d)-P_{\hat{\theta}}([-\Delta-r,-\Delta+r]^d) \nonumber\\
   \leq& \PX([-\Delta-r,-\Delta+r]^d)-\Upsilon([-\Delta-r,-\Delta+r]^d\times[-\Delta-r,-\Delta+r]^d)\nonumber\\
   =&\Upsilon(((\mathbb{R}^d\backslash\mathcal{K})\cup (\mathcal{K}\backslash\mathcal{K}_0)\cup[\Delta-r,\Delta+r]^d)\times [-\Delta-r,-\Delta+r]^d)\nonumber\\
   \leq& \Upsilon((\mathbb{R}^d\backslash\mathcal{K})\times[-\Delta-r,-\Delta+r]^d)+\Upsilon(
    ((\mathcal{K}\backslash \mathcal{K}_0)\cup [\Delta-r,\Delta+r]^d)
    \times [-\Delta-r,-\Delta+r]^d)\nonumber\\
 \leq& P_{\hat{\theta}}(\mathbb{R}^d\backslash\mathcal{K})+ \Upsilon(
    ((\mathcal{K}\backslash \mathcal{K}_0)\cup [\Delta-r,\Delta+r]^d)
    \times [-\Delta-r,-\Delta+r]^d)\nonumber\\
 \leq& d\psi(r)+8d\sqrt{M\max_{j\in[d]}\{\epsilon_j\}}+2d\sqrt{\frac{2\log(8d\slash \delta)}{n}}+\frac{\hat{\mathcal{R}}}{ 2(\Delta-r)}.
\end{align*}
Hence
\begin{align}\label{E4.22n}
    &|\PX([-\Delta-r,-\Delta+r]^d)-P_{\hat{\theta}}([-\Delta-r,-\Delta+r]^d)|\nonumber\\
    \leq& d\psi(r)+8d\sqrt{M\max_{j\in[d]}\{\epsilon_j\}}+2d\sqrt{\frac{2\log(8d\slash \delta)}{n}}+\frac{\hat{\mathcal{R}}}{ 2(\Delta-r)}.
\end{align}
By \eqref{E4.23} and \eqref{E4.22n},
\begin{align}
    &\big|\hat{w}_{2,r}-w_2\big|=\big|P_{\hat{\theta}}([-\Delta-r,-\Delta+r]^d)-w_2\big|\nonumber\\
    \leq& d\psi(r)+8d\sqrt{M\max_{j\in[d]}\{\epsilon_j\}}+5d\sqrt{\frac{\log(16 d\slash \delta)}{2n}}+\frac{\hat{\mathcal{R}}}{ 2(\Delta-r)}+\big|P_X([-\Delta-r,-\Delta+r]^d)-w_2\big|\nonumber\\
    \leq& 2d\psi(r)+8d\sqrt{M\max_{j\in[d]}\{\epsilon_j\}}+5d\sqrt{\frac{\log(16 d\slash \delta)}{2n}}+\frac{\hat{\mathcal{R}}}{ 2(\Delta-r)}.
\end{align}
Similarly, we have 
\begin{align}
    \big|\hat{w}_{1,r}-w_1\big| 
    \leq 2d\psi(r)+8d\sqrt{M\max_{j\in[d]}\{\epsilon_j\}}+5d\sqrt{\frac{\log(16 d\slash \delta)}{2n}}+\frac{\hat{\mathcal{R}}}{ 2(\Delta-r)}.
\end{align}

\end{proof}

\begin{proof}[Proof of Corollary 4.3.1]

For any $j\in [d]$, $\rho_j(x)=(2\pi\sigma^2)^{-1\slash 2}\exp(-x^2\slash (2\sigma^2))$ for all $x\in\mathbb{R}$. Hence for any $j\in [d]$ and $r\geq 0$,
\begin{align*}
    \sup_{x\in\mathbb{R}}\rho_j(x)\leq (2\pi\sigma^2)^{-1\slash 2}, \quad \rho_j([-r,r]^c)\leq 2\exp\Big(-\frac{r^2}{2\sigma^2}\Big).
\end{align*}
Thus we can take $M=(2\pi\sigma^2)^{-1\slash 2}$ and $\psi(r)=2\exp(-r^2\slash (2\sigma^2))$ in Theorem 4.3. By Theorem 4.3, with probability at least $1-\delta$, for all $k\in [2]$ and any $r\in (0,\Delta)$, we have 
\begin{align*}
    \big|\hat{w}_{k,r}-w_k\big|&\leq 2d\psi(r)+8d\sqrt{M\max_{j\in[d]}\{\epsilon_j\}}+5d\sqrt{\frac{\log(16 d\slash \delta)}{2n}}+\frac{\hat{\mathcal{R}}}{ 2(\Delta-r)}\nonumber\\
    &\leq 4d\exp\Big(-\frac{r^2}{2\sigma^2}\Big)+8d\sqrt{\frac{\max_{j\in[d]}\{\epsilon_j\}}{\sigma}}+5d\sqrt{\frac{\log(16 d\slash \delta)}{2n}}+\frac{\hat{\mathcal{R}}}{ 2(\Delta-r)}.
\end{align*}
    
\end{proof}

\begin{proof}[Proof of Theorem 4.4]

We fix any $\delta\in (0,1)$, and define $\mathcal{E}_{\delta}$ as in Lemma \ref{Lemma2.2}. By Lemma \ref{Lemma2.2}, $\mathbb{P}(\mathcal{E}_{\delta}^c)\leq \delta$. For the rest of the proof, we assume that the event $\mathcal{E}_{\delta}$ holds.

By the definition of $\hat{\theta}$ (the solution to the constrained minimization problem in equation (\textcolor{red}{4.3}) of the paper), for any $j\in [d]$, $W_1(\PXj,P_{\hat{\theta},j})\leq \epsilon_j$. Hence by Lemma \ref{Lem:KS}, for any $j\in [d]$, $t\in\mathbb{R}$, and $\gamma>0$,
\begin{equation*}
   P_{\hat{\theta},j}((-\infty,-t])\geq \PXj((-\infty,-t-\gamma])-\gamma^{-1}\epsilon_j.
\end{equation*}
Similarly,
\begin{equation*}
    P_{\hat{\theta},j}([t,\infty))\geq \PXj([t+\gamma,\infty))-\gamma^{-1}\epsilon_j.
\end{equation*}
Hence for any $j\in [d]$ and $t,\gamma>0$, 
\begin{equation}\label{E4.7}
    P_{\hat{\theta},j}((-t,t)^c)\geq  \PXj((-(t+\gamma),t+\gamma)^c)-2\gamma^{-1}\epsilon_j.
\end{equation}

Below we consider any sufficiently large $r$ such that
\begin{equation*}
    P_X(\{\mathbf{x}\in\mathbb{R}^d:\|\mathbf{x}\|_2\geq r\})\geq c_0\exp(-C_0 r^{\alpha}).
\end{equation*}
As
\begin{equation*}
    \{\mathbf{x}\in\mathbb{R}^d:\|\mathbf{x}\|_2\geq r\}\subseteq \bigcup_{j=1}^d \{\mathbf{x}\in\mathbb{R}^d:|x_j|\geq r d^{-1\slash 2}\},
\end{equation*}
we have 
\begin{align*}
  c_0\exp(-C_0 r^{\alpha})&\leq   P_X(\{\mathbf{x}\in\mathbb{R}^d:\|\mathbf{x}\|_2\geq r\})\leq P_X\Big(\bigcup_{j=1}^d \{\mathbf{x}\in\mathbb{R}^d:|x_j|\geq r d^{-1\slash 2}\}\Big)\nonumber\\
    &\leq \sum_{j=1}^d P_X(\{\mathbf{x}\in\mathbb{R}^d:|x_j|\geq r d^{-1\slash 2}\}).
\end{align*}
Hence there exists $j\in [d]$, such that
\begin{equation}\label{E4.8}
  P_{X,j}((-rd^{-1\slash 2}, rd^{-1\slash 2})^c) =  P_X(\{x\in\mathbb{R}^d:|x_j|\geq r d^{-1\slash 2}\})\geq c_0 d^{-1} \exp(-C_0 r^{\alpha}). 
\end{equation}
By \eqref{E4.7}-\eqref{E4.8} and Lemma \ref{Lemma2.2}, for any $\gamma\in (0,rd^{-1\slash 2})$, we have  
\begin{align}\label{E4.9}
  & P_{\hat{\theta}}(\{\mathbf{x}\in\mathbb{R}^d:\|\mathbf{x}\|_2\geq rd^{-1\slash 2}-\gamma\})\geq P_{\hat{\theta},j}((-rd^{-1\slash 2}+\gamma, rd^{-1\slash 2}-\gamma)^c)\nonumber\\
  \geq & \PXj((-rd^{-1\slash 2},rd^{-1\slash 2})^c)-2\gamma^{-1}\epsilon_j\nonumber\\
  \geq & P_{X,j}((-rd^{-1\slash 2},rd^{-1\slash 2})^c)-\sqrt{\frac{2\log(4d\slash \delta)}{n}}-2\gamma^{-1}\epsilon_j\nonumber\\
  \geq & c_0 d^{-1} \exp(-C_0 r^{\alpha})-\sqrt{\frac{2\log(4d\slash \delta)}{n}}-2\gamma^{-1}\max_{j\in [d]}\{\epsilon_j\}.
\end{align}
Taking $\gamma=1$, the conclusion of Theorem 4.4 follows from \eqref{E4.9} (taking $r$ sufficiently large).
    
\end{proof}

\section{Additional Experiment Results and Implementation Details}

\subsection{Optimizing the Generator Network}

The update rule for each parameter $\theta_i$ at iteration $t$ is given by:
\begin{align*}
\theta_i^{(t+1)} = \theta_i^{(t)} - \eta^{(t)} \left( \frac{\alpha \hat{m}_i^{(t)}}{\sqrt{\hat{v}_i^{(t)}} + \epsilon} + \lambda\theta_i^{(t)}\right)
\end{align*}
where $\alpha$ denotes the learning rate, $\eta^{(t)}$ is the scheduled learning rate multiplier, $\hat{m}_i^{(t)}$ and $\hat{v}_i^{(t)}$ are the bias-corrected first and second moment estimates, $\lambda$ is the weight decay factor, and $\epsilon$ is a small constant for numerical stability.
We also include a cosine annealing scheduler \citep{loshchilov2017sgdr} to modulate the learning rate; the cosine annealing schedule for the learning rate facilitates initial rapid exploration of the parameter space followed by gradual decrease of learning rate over the course of training. 

\subsection{Rapid Convergence of \potnet}
To demonstrate the rapid empirical convergence of \potnet, we generate a fixed set of 3,000 noise variables and propagate them through the generator as training progresses. 
The simulation setup is the same as that of the Bayesian Likelihood-Free Inference example. 
The results are illustrated in the figure below, where cyan represents the ground truth and orange depicts the generated samples. The figure clearly illustrates that \potnet converges around epoch 30 and maintains robustness without experiencing mode collapse in subsequent epochs.

\begin{figure}[H]
\begin{center}
    \includegraphics[width=0.45\textwidth]{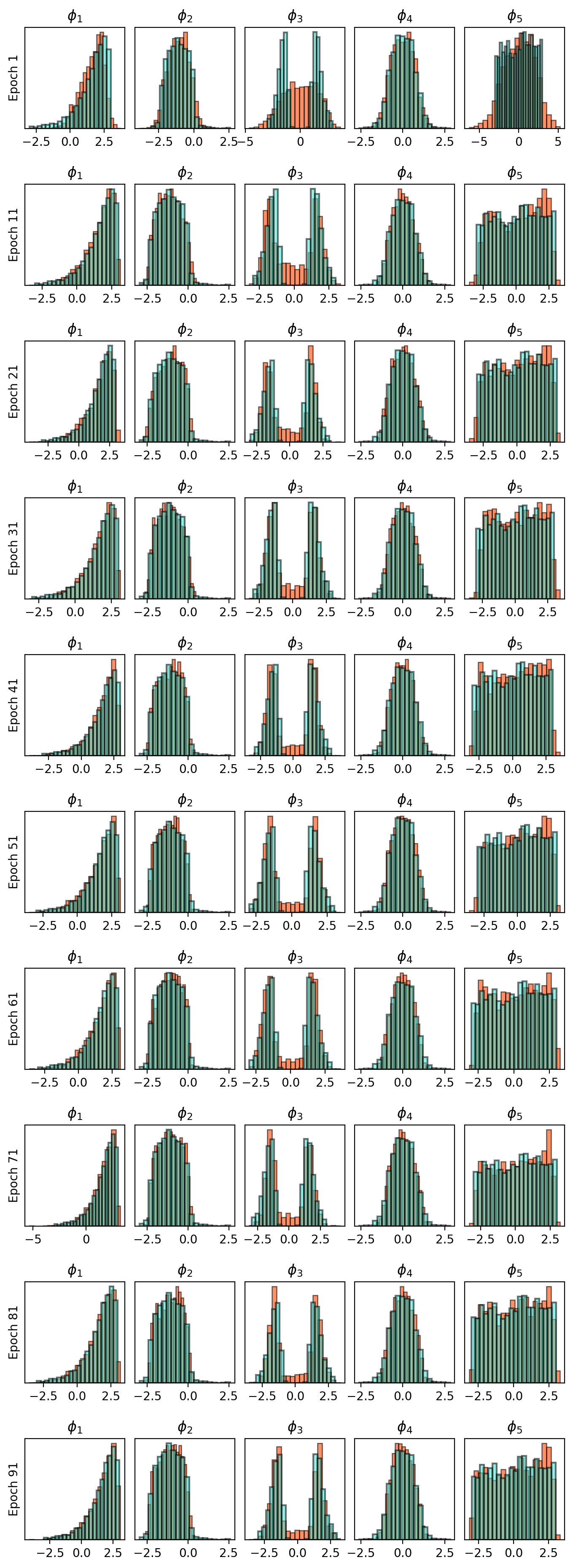}
\caption{Evolution of generated samples through training. Cyan: ground truth; orange: synthetic samples generated by \potnet. \potnet converges around epoch 30.}
\end{center}
\label{fig:potnet-rapid-converge}
\end{figure}

\subsection{Robustness of \potnet across Random Initializations}
To demonstrate \potnet's robustness to initialization, we trained ten independently initialized instances for 100 epochs each. 
The simulation setup is the same as that of the Bayesian Likelihood-Free Inference example. 
The figure below illustrates the generated samples across these initializations. The consistent alignment of orange synthetic samples with the cyan ground truth across all trials indicates \potnet's stability and insensitivity to initial conditions. This robust performance shows stability of \potnet in generating high-quality samples regardless of the random initialization.
\begin{figure}[H]
\begin{center}
    \includegraphics[width=0.3\textwidth]{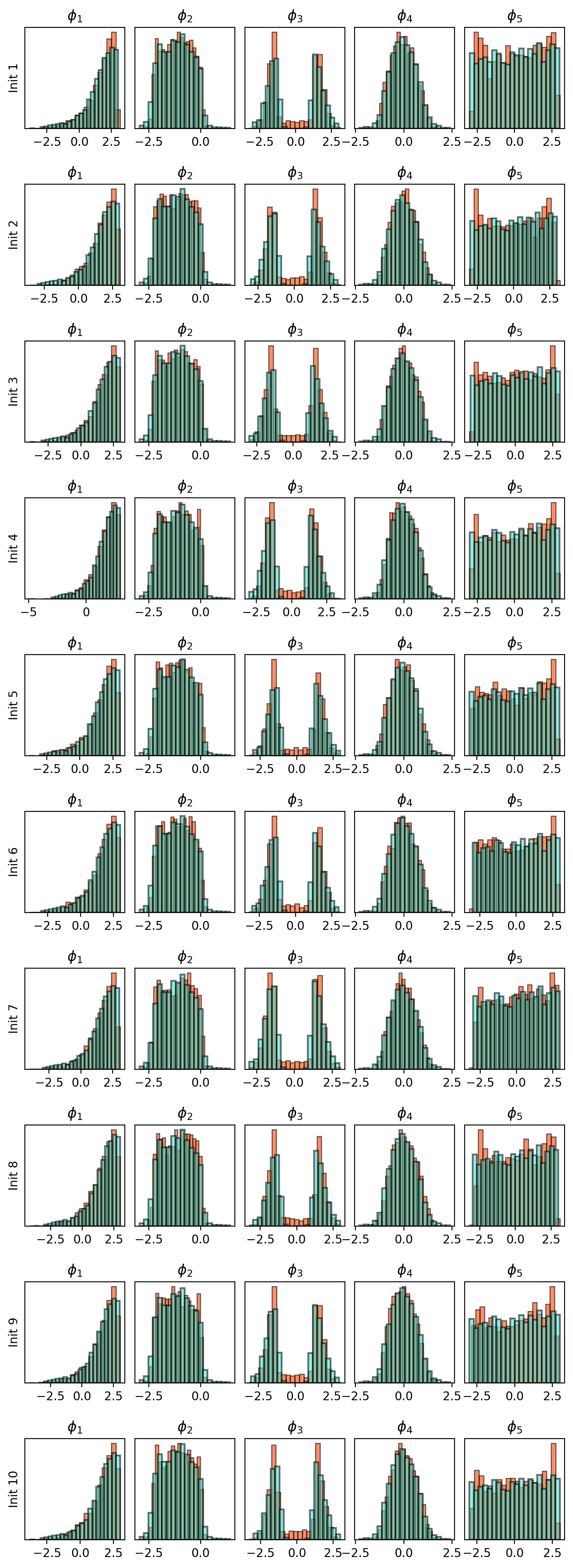}
\caption{Generated samples across 10 random initializations. Cyan: ground truth; orange: synthetic samples generated by \potnet. \potnet demonstrates stable and robust performance across random initializations.}
\end{center}
\label{fig:potnet-random-init}
\end{figure}

\subsection{Recursive Training with Synthetic Data}\label{appendix:recursive-training}

In this example, we employ the same simulation setup as in the Mixture of Gaussians experiment. The first iteration of training is identical to that of the Mixture of Gaussians experiment. 
Subsequently, we generate 2,000 synthetic data points using each model. These synthetic data points are then used to train a new generative model of the same type, constituting the second iteration of the process. The results are displayed in Figure 1.

Notably, both OT and SW methods propagate significant bias into Iteration 2. OT exhibits exacerbated mode collapse and fails to capture the full diversity of the data distribution (observe the shrinkage of the red countour relative to the blue contour). 
SW intensifies its distortion of modes (see the horizontal compression of the left mode). In contrast, synthetic data generated by \potnet maintains overall accuracy across iterations.

\begin{figure}[H]
\begin{center}
    \includegraphics[width=1\textwidth]{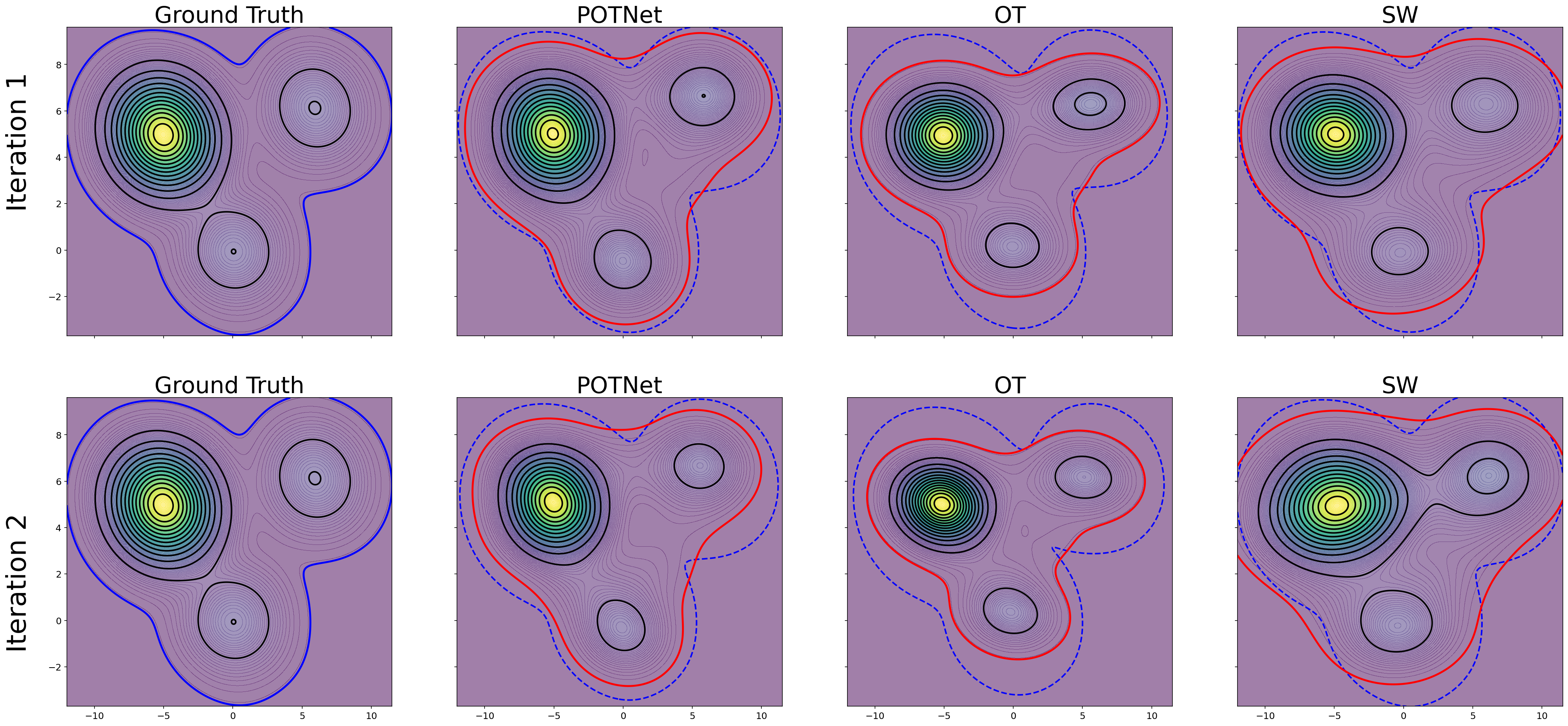}
\caption{Bivariate contour plots of the first two dimensions of a 20-dimensional Gaussian mixture model with three components, using synthetic data generated by recursive model training. Top panel: results from Iteration 1. Bottom panel: results from Iteration 2. Blue contours represent the tail distribution of the ground truth data. Red contours depict the tail distribution of the synthetic data.}
\end{center}
\label{Fig:gmm-seq-retrain}
\end{figure}

\subsection{Bayesian Likelihood-Free Inference}
\paragraph{Evaluation of total variation distance via empirical distribution}
We divide the space $[-3,3]^5$ into $3,125$ equally sized bins. For a set $P$ of synthetic samples, we estimate its empirical distribution using these bins:
\begin{equation*}
    \mathbb{Q}_P:=\sum_{x} P(x) \delta_x,
\end{equation*}
where the sum of $x$ is over the centers of the bins and $P(x)$ is the proportion of synthtic samples from $P$ that fall into the bin centered at $x$. For two sets $P,Q$ of synthetic samples, we evaluate the total variation distance between their empirical distributions by
\begin{equation*}
    d_{TV}(\mathbb{Q}_P,\mathbb{Q}_Q)=\frac{1}{2}\sum_x |P(x)-Q(x)|.
\end{equation*}

\paragraph{Visualization of mode collapse in three-dimensional space}
The plot below compares the generated samples in three dimensions, highlighting the pronounced type I and type II mode collapse behaviors exhibited by OT. While SW does not suffer from mode collapse, it fails to clearly distinguish between the two modes and thus produces in low-quality generated samples.
\begin{figure}[h!]
\begin{center}
    \includegraphics[width=.9\textwidth]{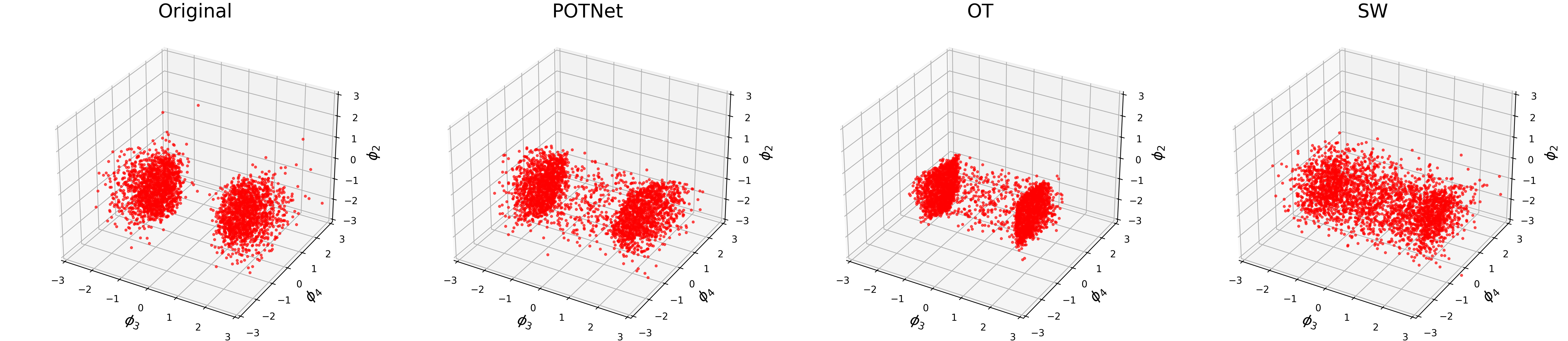}
\caption{3D comparison of generated samples. OT-generated samples exhibit severe shrinkage towards the mode of the cluster in the $\phi_3$ direction. SW-generated samples fail to produce distinct and distinguishable clusters. \potnet-generated samples captures cluster characteristics while preserving the cluster diversity.}
\end{center}
\label{fig:abc-3d}
\end{figure}

\subsection{Mixture of Gaussians}
For this example, the true data generating distribution is defined by component weights, $\mathbf{w} = (0.1, 0.1, 0.8)^\top$, means $\pmb{\mu}_1 = \mathbf{0}_{20}$, $\pmb{\mu}_2 = 6 \cdot \mathbf{1}_{20}$, and $\pmb{\mu}_3 = (-5, 5)_{10}^\top$, with covariance matrices,
\begin{align*}
\mathbf{\Sigma}_1 &= \mathbf{\Sigma}_2 = \begin{bmatrix}
1 & 0 & 0 & \cdots & 0 \\
0 & 1 & 0 & \cdots & 0 \\
0 & 0 & 1 & \cdots & 0 \\
\vdots & \vdots & \vdots & \ddots & \vdots \\
0 & 0 & 0 & \cdots & 1
\end{bmatrix}_{20 \times 20} \\[10pt]
\mathbf{\Sigma}_3 &= \begin{bmatrix}
1 & -0.09 & 0 & \cdots & 0 \\
-0.09 & 1 & 0 & \cdots & 0 \\
0 & 0 & 1 & \cdots & 0 \\
\vdots & \vdots & \vdots & \ddots & \vdots \\
0 & 0 & 0 & \cdots & 1
\end{bmatrix}_{20 \times 20}
\end{align*}
The real dataset consists of 2,000 observations. 
We use the following generator network architecture for \potnet, WGAN, OT, and SW,
\begin{align*}
    z \in \R^{20} &\to \mathrm{FC}_{500} \to \mathrm{BatchNorm1d} \to \mathrm{Dropout} \to \mathrm{ReLU} \\
    &\to \mathrm{FC}_{200} \to \mathrm{BatchNorm1d} \to \mathrm{Dropout} \to \mathrm{ReLU} \\
    &\to \mathrm{FC}_{100} \to \mathrm{BatchNorm1d} \to \mathrm{Dropout} \to \mathrm{ReLU} \to \mathrm{FC}_{20}
\end{align*}
and use the default network architecture for CTGAN.
For the critic network of WGAN, we use the following architecture:
\begin{align*}
    x \in \R^{20} &\to \mathrm{FC}_{500} \to \mathrm{LeakyReLU(0.2)} \\
    &\to \mathrm{FC}_{200} \to \mathrm{LeakyReLU(0.2)} \\
    &\to \mathrm{FC}_{1}
\end{align*}

Following convention, for WGAN with gradient penalty, the critic network is updated 5 times for every generator update and the gradient penalty parameter is set to 10.
For each method, we set the latent dimension equal to the number of features of 20.
Each method is trained for 200 epochs, with the exception of WGAN with gradient penalty, which is trained for 400 epochs to ensure convergence. 

\subsection{MNIST Digit Generation}
\paragraph{Setting 1: Convoluational Neural Network Architecture} 
Under this setting, the generator network architecture consists of a fully connected layer, followed by three sequential blocks. Each block comprises upsampling, 2D convolution, batch normalization, and leaky ReLU layers (with the exception of the final layer which utilizes hyperbolic tangent activation). This design progressively refines image features throughout the network.
We employ a batch size of 256 and a latent dimension of 100 for all networks.
The full generator network architecture is as follows:
\begin{align*}
    z \in \R^{100} &\to \mathrm{FC}_{128\times 7\times 7} \to \mathrm{BatchNorm2d} \to \mathrm{Upsample}(2.0) \\
    &\to \mathrm{Conv2d}_{128\times 128} \to \mathrm{BatchNorm2d} \to \mathrm{ReLU}\\
    & \to \mathrm{Upsample}(2.0) \to \mathrm{Conv2d}(128 \times 64) \to \mathrm{BatchNorm2d} \to \mathrm{ReLU} \\
    &\to \mathrm{Conv2d}(64 \times 1) \to \mathrm{Tanh}
\end{align*}
For WGAN, we update the critic network five times for every generator update and apply a gradient penalty of 10.
We adopt the following architecture for the critic network of WGAN:
\begin{align*}
    x \in \R^{784} &\to \mathrm{Conv2d}_{784 \times 64} \to \mathrm{LeakyReLU}(0.2) \\
    &\to \mathrm{Conv2d}(64 \times 128) \to \mathrm{LeakyReLU}(0.2) \\
    &\to \mathrm{Flatten} \to \mathrm{FC}_{1}
\end{align*}

\end{appendices}

\end{document}